\documentclass[nohyperref]{article}

\usepackage{microtype}
\usepackage{graphicx}
\usepackage{subfigure}
\usepackage{booktabs} % for professional tables

\usepackage{hyperref}
\usepackage{enumitem}
\usepackage{multirow}

\usepackage{graphicx}
\usepackage{amssymb}
\usepackage{amsmath}
\usepackage{xfrac}
\usepackage{epstopdf}
\usepackage{color}
\usepackage{mathtools}
\usepackage{bbm}
\usepackage{hyperref}

\newcommand{\y}{\mathbf{y}}
\newcommand{\h}{\mathbf{H}}
\newcommand{\E}{\mathbb{E}}
\newcommand{\I}{\mathbb{I}}
\newcommand{\R}{\mathbb{R}}

\newcommand{\inner}[2]{\left\langle #1, #2\right\rangle}

\usepackage[accepted]{icml2022}

\usepackage{amsmath}
\usepackage{amssymb}
\usepackage{mathtools}
\usepackage{amsthm}

\usepackage[capitalize,noabbrev]{cleveref}

\theoremstyle{plain}
\newtheorem{theorem}{Theorem}
\theoremstyle{definition}

\theoremstyle{remark}

\newtheorem{lemma}{Lemma}

\newtheorem{asump}{Assumption}
\newtheorem{property}{Property}

\usepackage[textsize=tiny]{todonotes}

\icmltitlerunning{\texttt{GIST}: Distributed Training for Large-Scale Graph Convolutional Networks}

\begin{document}

\twocolumn[
\icmltitle{\texttt{GIST}: Distributed Training for Large-Scale Graph Convolutional Networks}



\icmlsetsymbol{equal}{*}

\begin{icmlauthorlist}
\icmlauthor{Cameron R. Wolfe}{equal,cs}
\icmlauthor{Jingkang Yang}{equal,ntu}
\icmlauthor{Fangshuo Liao}{equal,cs}
\icmlauthor{Arindam	Chowdhury}{ece}
\icmlauthor{Chen Dun}{cs}
\icmlauthor{Artun Bayer}{ece}
\icmlauthor{Santiago Segarra}{ece}
\icmlauthor{Anastasios Kyrillidis}{cs}
\end{icmlauthorlist}

\icmlaffiliation{ntu}{School of Computer Science and Engineering, Nanyang Technology University, Singapore.}
\icmlaffiliation{cs}{Department of Computer Science, Rice University, Houston, TX, USA.}
\icmlaffiliation{ece}{Department of Electrical and Computer Engineering, Rice University, Houston, TX, USA.}

\icmlcorrespondingauthor{Cameron Wolfe}{crw13@rice.edu}

\icmlkeywords{Graph Neural Networks, Graph Convolutional Networks, Distributed Training}

\vskip 0.3in
]



\printAffiliationsAndNotice{\icmlEqualContribution} % otherwise use the standard text.

\begin{abstract}
The graph convolutional network (GCN) is a go-to solution for machine learning on graphs, but its training is notoriously difficult to scale both in terms of graph size and the number of model parameters.
Although some work has explored training on large-scale graphs (e.g., GraphSAGE, ClusterGCN, etc.), we pioneer efficient training of large-scale GCN models (i.e., ultra-wide, overparameterized models) with the proposal of a novel, distributed training framework.
Our proposed training methodology, called \texttt{GIST}, disjointly partitions the parameters of a GCN model into several, smaller sub-GCNs that are trained independently and in parallel.
In addition to being compatible with all GCN architectures and existing sampling techniques for efficient GCN training, \texttt{GIST} $i)$ improves model performance, $ii)$ scales to training on arbitrarily large graphs, $iii)$ decreases wall-clock training time, and $iv)$ enables the training of markedly overparameterized GCN models.
Remarkably, with \texttt{GIST}, we train an astonishgly-wide $32,\!768$-dimensional GraphSAGE model, which exceeds the capacity of a single GPU by a factor of $8\times$, to SOTA performance on the Amazon2M dataset. 
\end{abstract}

\section{Introduction} \label{S:intro}
\begin{figure}
    \centering
    \includegraphics[width=0.8\linewidth]{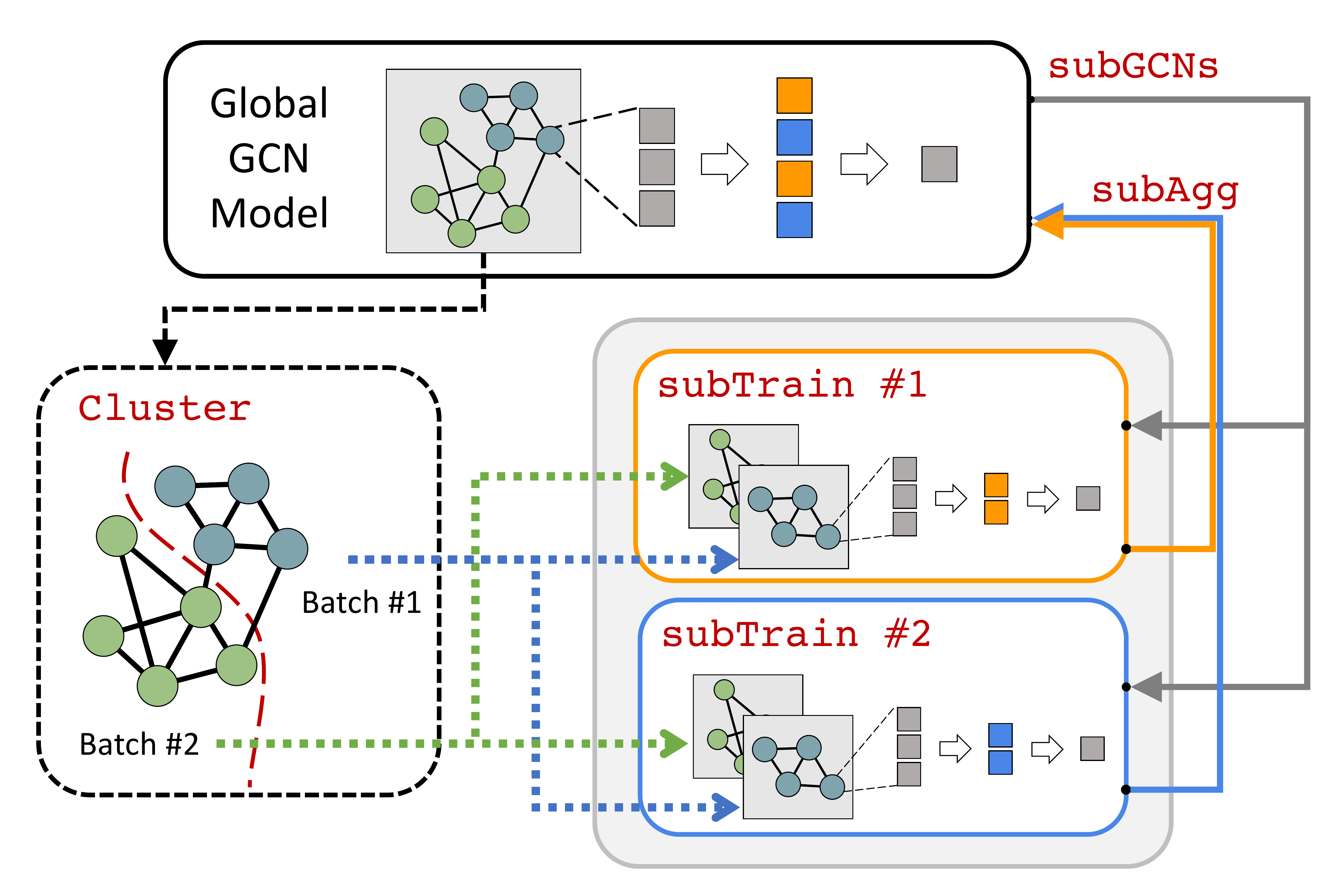}
    \vspace{-0.3cm}
    \caption{\texttt{GIST} pipeline: \texttt{subGCNs} divides the global GCN into sub-GCNs. Every sub-GCN is trained by \texttt{subTrain} using mini-batches (smaller sub-graphs) generated by \texttt{Cluster}. Sub-GCN parameters are intermittently aggregated through \texttt{subAgg}.}
    \vspace{-0.5cm}
    \label{fig:clustergcnist}
\end{figure}

Since not all data can be represented in Euclidean space \cite{geometric_dl}, many applications rely on graph-structured data.
For example, social networks can be modeled as graphs by regarding each user as a node and friendship relations as edges~\cite{lusher2013exponential, newman2002random}.
Alternatively, in chemistry, molecules can be modeled as graphs, with nodes representing atoms and edges encoding chemical bonds~\cite{balaban1985applications,benko2003graph}.

To better understand graph-structured data, several (deep) learning techniques have been extended to the graph domain~\cite{defferrard2016convolutional, gori2005new,masci2015geodesic}.
Currently, the most popular one is the graph convolutional network (GCN)~\cite{gcn}, a multi-layer architecture that implements a generalization of the convolution operation to graphs.
Although the GCN handles node- and graph-level classification, it is notoriously inefficient and unable to handle large-scale graphs~\cite{vrgcn, fastgcn, lsgcn, fastrep, l2gcn, graphsaint}. 

To deal with these issues, node partitioning methodologies have been developed. 
These schemes can be roughly categorized into neighborhood sampling~\cite{ fastgcn, graphsage, ladies} and graph partitioning~\cite{clustergcn, graphsaint} approaches.
The goal is to partition a large graph into multiple smaller graphs that can be used as mini-batches for training the GCN.
In this way, GCNs can handle larger graphs during training, expanding their potential into the realm of big data.

Although some papers perform large-scale experiments~\cite{clustergcn, graphsaint}, the models (and data) used in GCN research remain small in the context of deep learning~\cite{gcn, gat}, where the current trend is towards incredibly large models and datasets \cite{gpt3, xlm-r}.
Despite the widespread moral questioning of this trend \cite{transcarbonemit, sotaaimodels, costofnlp}, the deep learning community continues to push the limits of scale, as overparameterized models are known to discover generalizable solutions \cite{doubledescent}.
Although deep GCN models suffer from oversmoothing \cite{gcn, oversmooth}, overparameterized GCN models can still be explored through larger hidden layers.
\textit{As such, this work aims to provide a training framework that enables GCN experiments with wider models and larger datasets.}

\textbf{This paper.}
We propose a novel, distributed training methodology that can be used for any GCN architecture and is compatible with existing node sampling techniques.
This methodology randomly partitions the hidden feature space in each layer, decomposing the global GCN model into multiple, narrow sub-GCNs of equal depth.
Sub-GCNs are trained independently for several iterations in parallel prior to having their updates synchronized; see Figure \ref{fig:clustergcnist}. 
This process of randomly partitioning, independently training, and synchronizing sub-GCNs is repeated until convergence.
We call this method graph independent subnetwork training (\texttt{GIST}).
% By performing graph partitioning on larger datasets, \texttt{GIST} can easily scale to arbitrarily large graphs.
\texttt{GIST} can easily scale to arbitrarily large graphs and significantly reduces the wall-clock time of training large-scale GCNs, allowing larger models and datasets to be explored.
We focus specifically on enabling the training of ``ultra-wide'' GCNs (i.e., GCN models with very large hidden layers), as deeper GCNs are prone to oversmoothing \cite{oversmooth}.
The contributions of this work are summarized below: \vspace{-0.3cm}
\begin{itemize}[leftmargin=0.5cm]
    \item We develop a novel, distributed training methodology for arbitrary GCN architectures, based on decomposing the model into independently-trained sub-GCNs. This methodology is compatible with existing techniques for neighborhood sampling and graph partitioning. \vspace{-0.2cm}
    \item We show that \texttt{GIST} can be used to train several GCN architectures to state-of-the-art performance with reduced training time in comparison to standard methodologies. \vspace{-0.2cm}
    \item We propose a novel Graph Independent Subnetwork Training Kernel (\texttt{GIST-K}) that allows a convergence rate to be derived for two-layer GCNs trained with \texttt{GIST} in the infinite width regime. Based on \texttt{GIST-K}, we provide theory that \texttt{GIST} converges linearly, up to an error neighborhood, using distributed gradient descent with local iterations. We show that the radius of the error neighborhood is controlled by the overparameterization parameter, as well as the number of workers in the distributed setting. Such findings reflect practical observations that are made in the experimental section. \vspace{-0.2cm}
    \item We use \texttt{GIST} to enable the training of markedly overparameterized GCN models. In particular, \texttt{GIST} is used to train a \textbf{two-layer GraphSAGE model with a hidden dimension of 32,768} on the Amazon2M dataset. \emph{Such a model exceeds the capacity of a single GPU by $8\times$.} \vspace{-0.2cm}
\end{itemize}

\begin{figure*}
    \centering
    \includegraphics[width=0.9\linewidth]{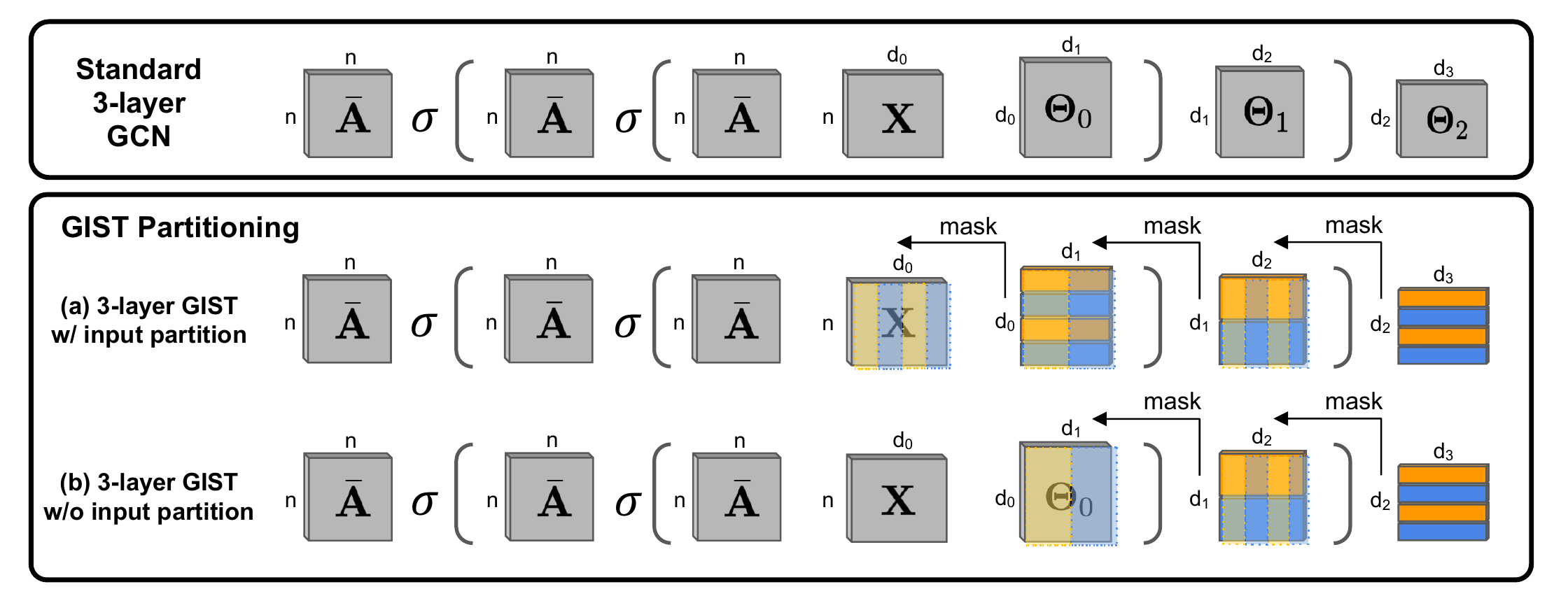}
    \vspace{-0.3cm}
    \caption{GCN partition into $m=2$ sub-GCNs. Orange and blue colors depict different feature partitions. Both hidden dimensions ($d_1$ and $d_2$) are partitioned. The output dimension ($d_3$) is not partitioned. Partitioning the input dimension ($d_0$) is optional. In this work, we do not partition $\mathbf{d}_0$ in \texttt{GIST}.}
    \label{fig:main}
\end{figure*}

\vspace{-0.1cm}
\section{What is the \texttt{GIST} of this work?}
\label{S:method}

\begin{algorithm}[tb]
    \centering
    \caption{\texttt{GIST} Algorithm}\label{alg:gist}
        \begin{algorithmic}
        \STATE \textbf{Parameters}: $T$ synchronization iterations, $m$ sub-GCNs \\ $\zeta$ local iterations, $c$ clusters, $\mathcal{G}$ training graph. \vspace{-2mm}
        \\\hrulefill 
        \STATE $\Psi_{\mathcal{G}}(\, \cdot \, ; \boldsymbol{\Theta})$ $\leftarrow$ randomly initialize GCN
        \STATE $\{\mathcal{G}_{(j)}\}_{j=1}^c \leftarrow \text{\texttt{Cluster}}(\mathcal{G}, c)$ 
        \FOR{$t = 0, \dots, T-1$}
        \STATE$\big\{\Psi_{\mathcal{G}}( \, \cdot \, ; \boldsymbol{\Theta}^{(i)})\big\}_{i = 1}^m \leftarrow \text{\texttt{subGCNs}}(\Psi_{\mathcal{G}}(\, \cdot \, ; \boldsymbol{\Theta}) , m)$
        \STATE Distribute each $\Psi_{\mathcal{G}}( \, \cdot \, ; \boldsymbol{\Theta}^{(i)})$ to a different worker
        \FOR{$i = 1, \dots, m$}
            \FOR{$z = 1, \dots, \zeta$}
            \STATE $\Psi_{\mathcal{G}}( \, \cdot \, ; \boldsymbol{\Theta}^{(i)}) \leftarrow \text{\texttt{subTrain}}(\boldsymbol{\Theta}^{(i)}, \{\mathcal{G}_{(j)}\}_{j=1}^c)$
            \ENDFOR
        \ENDFOR
        \STATE $\Psi_{\mathcal{G}}(\, \cdot \, ; \boldsymbol{\Theta}) \leftarrow \text{\texttt{subAgg}}(\{\boldsymbol{\Theta}^{(i)}\}_{i=1}^m) $
    \ENDFOR \\
    \end{algorithmic}
\end{algorithm}

\textbf{GCN Architecture}. The GCN~\cite{gcn} is arguably the most widely-used neural network  architecture on graphs.
Consider a graph $\mathcal{G}$ comprised of $n$ nodes with $d$-dimensional features $\mathbf{X} \in \mathbb{R}^{n \times d}$.
The output $\mathbf{Y} \in \mathbb{R}^{n \times d'}$ of a GCN can be expressed as $\mathbf{Y} = \Psi_{\mathcal{G}}(\mathbf{X}; \boldsymbol{\Theta})$, where $\Psi_{\mathcal{G}}$ is an $L$-layered architecture with trainable parameters $\boldsymbol{\Theta}$.
If we define $\mathbf{H}_0 = \mathbf{X}$, we then have that  $\mathbf{Y} = \Psi_{\mathcal{G}}(\mathbf{X}; \boldsymbol{\Theta}) = \mathbf{H}_L$, where an intermediate $\ell$-th layer of the GCN is given by
\begin{align}\label{gcn_forward}
    \mathbf{H}_{\ell+1} = \sigma(\bar{\mathbf{A}} \, \mathbf{H}_\ell \, \boldsymbol{\Theta}_\ell).
\end{align}

In~\eqref{gcn_forward}, $\sigma$ is an elementwise activation function (e.g., ReLU), $\bar{\mathbf{A}}$ is the degree-normalized adjacency matrix of $\mathcal{G}$ with added self-loops, and the trainable parameters $\boldsymbol{\Theta} = \{\boldsymbol{\Theta}_\ell\}_{\ell=0}^{L-1}$ have dimensions $\boldsymbol{\Theta}_\ell \in \mathbb{R}^{d_{\ell} \times d_{\ell+1}}$ with $d_0  = d$ and $d_L = d'$.
In Figure~\ref{fig:main}~(top), we illustrate nested GCN layers for $L=3$, but our methodology extends to arbitrary $L$.
The activation function of the last layer is typically the identity or softmax transformation -- we omit this in Figure~\ref{fig:main} for simplicity.

\textbf{\texttt{GIST} overview.} 
We overview \texttt{GIST} in Algorithm~\ref{alg:gist} and present a schematic depiction in Figure~\ref{fig:clustergcnist}.
We partition our (randomly initialized) global GCN into $m$ smaller, disjoint sub-GCNs with the \texttt{subGCNs} function ($m=2$ in Figures~\ref{fig:main} and \ref{fig:clustergcnist}) by sampling the feature space at each layer of the GCN; see Section~\ref{S:submodel}.
Each sub-GCN is assigned to a different worker (i.e., a different GPU) for $\zeta$ rounds of distributed, independent training through \texttt{subTrain}.
Then, newly-learned sub-GCN parameters are aggregated (\texttt{subAgg}) into the global GCN model.
This process repeats for $T$ iterations.
Our graph domain is partitioned into $c$ sub-graphs through the \texttt{Cluster} function ($c=2$ in Figure~\ref{fig:clustergcnist}).
This operation is only relevant for large graphs ($n>50,\!000$), and we omit it ($c=1$) for smaller graphs that don't require partitioning.\footnote{Though any clustering method can be used, we advocate the use of METIS~\cite{karypis1998fast, karypis1998multilevelk} due to its proven efficiency in large-scale graphs.}

\subsection{\texttt{subGCNs}: Constructing Sub-GCNs} \label{S:submodel}

\texttt{GIST} partitions a global GCN model into several narrower sub-GCNs of equal depth. 
Formally, consider an arbitrary layer $\ell$ and a random, disjoint partition of the feature set $[d_\ell] = \{1, 2, \ldots, d_\ell\}$ into $m$ equally-sized blocks $\{\mathcal{D}^{(i)}_\ell\}_{i=1}^m$.\footnote{For example, if $d_\ell = 4$ and $m=2$, one valid partition would be given by $\mathcal{D}^{(1)}_\ell = \{1,4\}$ and $\mathcal{D}^{(2)}_\ell = \{2,3\}$.}
Accordingly, we denote by $\boldsymbol{\Theta}^{(i)}_{\ell} = [\boldsymbol{\Theta}_{\ell}]_{\mathcal{D}^{(i)}_\ell \times \mathcal{D}^{(i)}_{\ell+1}}$ the matrix obtained by selecting from $\boldsymbol{\Theta}_{\ell}$ the rows and columns given by the $i$th blocks in the partitions of $[d_\ell]$ and $[d_{\ell+1}]$, respectively.
With this notation in place, we can define $m$ different sub-GCNs $\mathbf{Y}^{(i)} = \Psi_{\mathcal{G}}(\mathbf{X}^{(i)}; \boldsymbol{\Theta}^{(i)}) = \mathbf{H}^{(i)}_{L}$ where $\mathbf{H}^{(i)}_{0} = \mathbf{X}_{[n] \times \mathcal{D}^{(i)}_0}$ and each layer is given by:
\begin{align}\label{sub_gcn_forward}
    \mathbf{H}^{(i)}_{\ell+1} = \sigma(\bar{\mathbf{A}} \, \mathbf{H}^{(i)}_{\ell} \, \boldsymbol{\Theta}^{(i)}_{\ell}).
\end{align}

Sub-GCN partitioning is illustrated in Figure~\ref{fig:main}-(a), where $m=2$. 
Partitioning the input features is optional (i.e., (a) vs. (b) in Figure~\ref{fig:main}).
\emph{We do not partition the input features within \texttt{GIST}} so that sub-GCNs have identical input information (i.e., $\mathbf{X}^{(i)} = \mathbf{X}$ for all $i$); see Section \ref{S:small_scale}.
Similarly, we do not partition the output feature space to ensure that the sub-GCN output dimension coincides with that of the global model, thus avoiding any need to modify the loss function.
This decomposition procedure (\texttt{subGCNs} in Algorithm~\ref{alg:gist}) extends to arbitrary $L$.

\subsection{\texttt{subTrain}: Independently Training Sub-GCNs}
\label{S:sub_training}

Assume $c=1$ so that the \texttt{Cluster} operation in Algorithm~\ref{alg:gist} is moot and $\{\mathcal{G}_{(j)}\}_{j=1}^c = \mathcal{G}$.
Because $\mathbf{Y}^{(i)}$ and $\mathbf{Y}$ share the same dimension, sub-GCNs can be trained to minimize the same global loss function.
One application of \texttt{subTrain} in Algorithm~\ref{alg:gist} corresponds to a single step of stochastic gradient descent (SGD).
Inspired by local SGD \cite{use_local_sgd}, multiple, independent applications of \texttt{subTrain} are performed in parallel (i.e., on separate GPUs) for each sub-GCN prior to aggregating weight updates. 
The number of independent training iterations between synchronization rounds, referred to as local iterations, is denoted by $\zeta$, and the total amount of training is split across sub-GCNs.\footnote{For example, if a global model is trained on a single GPU for 10 epochs, a comparable experiment for \texttt{GIST} with two sub-GCNs would train each sub-GCN for only 5 epochs.}
Ideally, the number sub-GCNs and local iterations should be increased as much as possible to minimize communication and training costs.
In practice, however, such benefits may come at the cost of statistical inefficiency; see Section~\ref{S:small_scale}.

If $c > 1$, \texttt{subTrain} first selects one of the $c$ subgraphs in $\{\mathcal{G}_{(j)}\}_{j=1}^c$ to use as a mini-batch for SGD.
Alternatively, the union of several sub-graphs in $\{\mathcal{G}_{(j)}\}_{j=1}^c$ can be used as a mini-batch for training.
Aside from using mini-batches for each SGD update instead of the full graph, \emph{the use of graph partitioning does not modify the training approach outlined above.}
Some form of node sampling must be adopted to make training tractable when the full graph is too large to fit into memory.
However, both graph partitioning and layer sampling are compatible with GIST (see Sections \ref{S:large_scale} and \ref{S:gist_layer_samp}).
We adopt graph sampling in the main experiments due to the ease of implementation.
The novelty of our work lies in the feature partitioning strategy of \texttt{GIST} for distributed training, which is an orthogonal technique to node sampling; see Section \ref{S:value}.
% This pipeline, which allows our methodology to scale to arbitrarily large graphs, is illustrated in Figure~\ref{fig:clustergcnist} for $c=2$ clusters and $m=2$ sub-GCNs.
% Pre-processing the training graph into smaller sub-graphs during training loosely resembles the \texttt{ClusterGCN}~\cite{clustergcn} methodology.
% However, \emph{the focus of this paper is upon developing a distributed training methodology, not developing novel graph partitioning strategies.}
% We simply adopt existing graph partitioning methods to ease scalability to larger graphs.

After each sub-GCN completes $\zeta$ training iterations, their updates are aggregated into the global model (i.e., \texttt{subAgg} function in Algorithm~\ref{alg:gist}).
Within \texttt{subAgg}, each worker replaces global parameter entries $\boldsymbol{\Theta}$ with its own parameters $\boldsymbol{\Theta}^{(i)}$, where no collisions occur due to the disjointness of sub-GCN partitions.
Interestingly, not every parameter in the global GCN model is updated by \texttt{subAgg}.
For example, focusing on $\boldsymbol{\Theta}_{1}$ in Figure~\ref{fig:main}-(a), one worker will be assigned $\boldsymbol{\Theta}^{(1)}_{1}$ (i.e., overlapping orange blocks), while the other worker will be assigned $\boldsymbol{\Theta}^{(2)}_{1}$ (i.e., overlapping blue blocks).
The rest of $\boldsymbol{\Theta}_{1}$ is not considered within \texttt{subAgg}.
Nonetheless, since sub-GCN partitions are randomly drawn in each cycle $t$, one expects all of $\boldsymbol{\Theta}$ to be updated multiple times if $T$ is sufficiently large.

\subsection{What is the value of \texttt{GIST}?} \label{S:value}

\begin{figure*}
    \centering
    \includegraphics[width=0.85\linewidth]{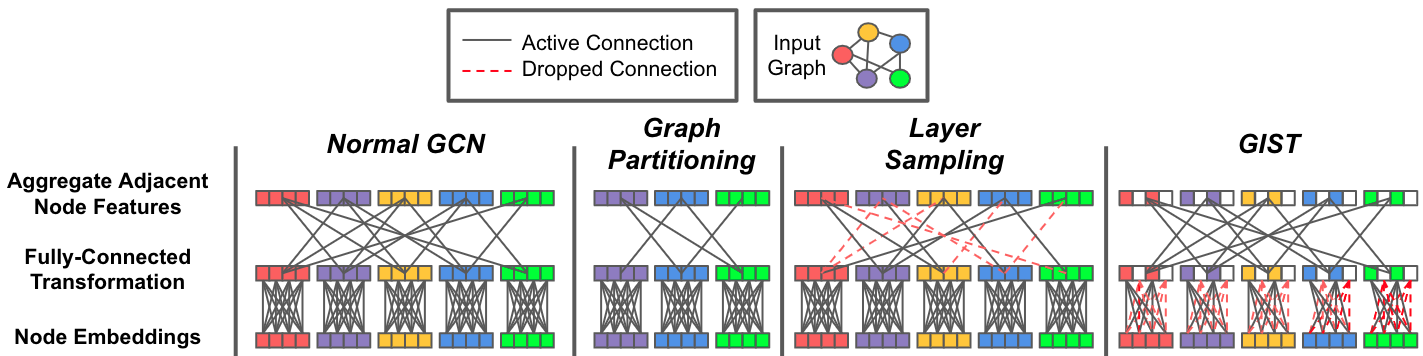}
    \caption{Illustrates the difference between \texttt{GIST} and node sampling techniques within the forward pass of a single GCN layer (excluding non-linear activation). While graph partitioning and layer sampling remove nodes from the forward pass (i.e., either completely or on a per-layer basis), \texttt{GIST} partitions node feature representations (and, in turn, model parameters) instead of the nodes themselves.}
    \label{fig:gist_v_other}
\end{figure*}

\noindent
\textbf{Architecture-Agnostic Distributed Training.}
\texttt{GIST} is a generic, distributed training methodology that can be used for any GCN architecture.
We implement \texttt{GIST} for vanilla GCN, GraphSAGE, and GAT architectures, but \texttt{GIST} is not limited to these models;
see Section \ref{S:experiment}.

\noindent
\textbf{Compatibility with Sampling Methods.}
\texttt{GIST} is \underline{\textbf{NOT}} a replacement for graph or layer sampling.
Rather, it is an efficient, distributed training technique that can be used in tandem with node partitioning.
As depicted in Figure \ref{fig:gist_v_other}, \texttt{GIST} partitions node feature representations and model parameters between sub-GCNs, while graph partitioning and layer sampling sub-sample nodes within the graph.

Interestingly, we find that \texttt{GIST}'s feature and parameter partitioning strategy is compatible with node partitioning---the two approaches can be combined to yield further efficiency benefits.
For example, \texttt{GIST} is combined with graph partitioning strategies in Section \ref{S:large_scale} and with layer sampling methodologies in Section \ref{S:gist_layer_samp}.
% Similarly, layer sampling methodologies \cite{vrgcn, fastgcn, ladies} are compatible with \texttt{GIST}; see experiments with LADIES \cite{ladies} in Section \ref{S:other_dist}.

\noindent
\textbf{Enabling Ultra-Wide GCN Training.}
\texttt{GIST} indirectly updates the global GCN through the training of smaller sub-GCNs, enabling models with hidden dimensions that exceed the capacity of a single GPU by a factor of $8\times$ to be trained. 
In this way, \texttt{GIST} allows markedly overparametrized (``ultra-wide") GCN models to be trained on existing hardware.
In Section~\ref{S:large_scale}, \emph{we leverage this capability to train a two-layer GCN model with a hidden dimension of 32,768 on Amazon2M.}

We argue that overparameterization through width is more valuable than overparameterization through depth because deeper GCNs could suffer from oversmoothing~\cite{oversmooth}.
As such, we do not explore depth-wise partitions of different GCN layers to each worker, but rather focus solely upon partitioning the hidden neurons within each layer.
Such a partitioning strategy is suited to training wider networks.

\textbf{Improved Model Complexity.}
Consider a single GCN layer, trained over $M$ machines with input and output dimension of $d_{i-1}$ and $d_{i}$, respectively.
For one synchronization round, the communication complexity of \texttt{GIST} and standard distributed training is $\mathcal{O}(\frac{1}{M}d_i d_{i-1})$ and $\mathcal{O}(M d_{i} d_{i-1})$, respectively.
\texttt{GIST} reduces communication by only communicating sub-GCN parameters. 
Existing node partitioning techniques cannot similarly reduce communication complexity because model parameters are never partitioned.
Furthermore, the computational complexity of the forward pass for a GCN model trained with \texttt{GIST} and using standard methodology is $\mathcal{O}(\frac{1}{M} N^2 d_i + \frac{1}{M^2} N d_i d_{i-1})$ and $\mathcal{O}(N^2 d_i + N d_i d_{i-1})$, respectively, where $N$ is the number of nodes in the partition being processed.\footnote{We omit the complexity of applying the element-wise activation function for simplicity.}
Node partitioning can reduce $N$ by a constant factor but is compatible with \texttt{GIST}.

% \subsection{Implementation Details}
% We provide an implementation of \texttt{GIST} in PyTorch~\cite{pytorch} using the NCCL distributed communication package for training GCN \cite{gcn}, GraphSAGE~\cite{graphsage} and GAT \cite{gat} architectures.
% Our implementation is centralized, meaning that a single process serves as a central parameter server.
% From this central process, the weights of the global model are maintained and partitioned to different worker processes (including itself) for independent training.
% Experiments are conducted with $8$ NVIDIA Tesla V100-PCIE-32G GPUs, a 56-core Intel(R) Xeon(R) CPU E5-2680 v4 @ 2.40GHz, and 256~GB of RAM.

\section{Related Work}

\textbf{GCN training.} 
In spite of their widespread success in several graph related tasks, GCNs often suffer from training inefficiencies~\cite{lsgcn,fastrep}.  
Consequently, the research community has focused on developing efficient and scalable algorithms for training GCNs~\cite{vrgcn,fastgcn,clustergcn,graphsage,graphsaint,ladies}.
The resulting approaches can be divided roughly into two areas: \emph{neighborhood sampling} and \emph{graph partitioning}. 
However, it is important to note that these two broad classes of solutions are not mutually exclusive, and reasonable combinations of the two approaches may be beneficial.

Neighborhood sampling methodologies aim to sub-select neighboring nodes at each layer of the GCN, thus limiting the number of node representations in the forward pass and mitigating the exponential expansion of the GCNs receptive field. 
VRGCN~\cite{vrgcn} implements a variance reduction technique to reduce the sample size in each layer, which achieves good performance with smaller graphs. 
However, it requires to store all the intermediate node embeddings during training, leading to a memory complexity close to full-batch training.
GraphSAGE~\cite{graphsage} learns a set of aggregator functions to gather information from a node's local neighborhood. It then concatenates the outputs of these aggregation functions with each node's own representation at each step of the forward pass.
FastGCN~\cite{fastgcn} adopts a Monte Carlo approach to evaluate the GCN's forward pass in practice, which computes each node's hidden representation using a fixed-size, randomly-sampled set of nodes.
LADIES~\cite{ladies} introduces a layer-conditional approach for node sampling, which encourages node connectivity between layers in contrast to FastGCN~\cite{fastgcn}.

Graph partitioning schemes aim to select densely-connected sub-graphs within the training graph, which can be used to form mini-batches during GCN training.
Such sub-graph sampling reduces the memory footprint of GCN training, thus allowing larger models to be trained over graphs with many nodes.
ClusterGCN~\cite{clustergcn} produces a very large number of clusters from the global graph, then randomly samples a subset of these clusters and computes their union to form each sub-graph or mini-batch.
Similarly, GraphSAINT~\cite{graphsaint} randomly samples a sub-graph during each GCN forward pass.
However, GraphSAINT also considers the bias created by unequal node sampling probabilities during sub-graph construction, and proposes normalization techniques to eliminate this bias.

As explained in Section~\ref{S:method}, \texttt{GIST} also relies on graph partitioning techniques (\texttt{Cluster}) to handle large graphs.
However, the feature sampling scheme at each layer (\texttt{subGCNs}) that leads to parallel and narrower sub-GCNs is a hitherto unexplored framework for efficient GCN training.

\textbf{Distributed training.} Distributed training is a heavily studied topic~\cite{distrib_comm, distrib_dl}.
Our work focuses on synchronous and distributed training techniques~\cite{decentralizedparallelsgd, layeredsgd, easgd}.
Some examples of synchronous, distributed training approaches include data parallel training, parallel SGD~\cite{parallelsgdanalysis, parallelsgd}, and local SGD~\cite{use_local_sgd, localsgdconverge}.
Our methodology holds similarities to model parallel training techniques, which have been heavily explored~\cite{parallelism_survey, integrated_model_parallel, layer_parallel_resnet, nonlinear_multigrid_layer_parallel,kfac, compiler_model_parallel,  lamp}.
More closely, our approach is inspired by independent subnetwork training~\cite{ist}, explored for multi-layer perceptrons. 

\section{Theoretical Results} \label{S:theory}
We draw upon analysis related to neural tangent kernels (NTK) \cite{jacot2018neural} to derive a convergence rate for two-layer GCNs using gradient descent---as formulated in \eqref{gcn_forward} and further outlined in Appendix \ref{A:theory_prelim}---trained with \texttt{GIST}.
% Namely, we adopt a two-layer GCN as formulated in \eqref{gcn_forward}.
Given the scaled Gram matrix of an infinite-dimensional NTK $\h^\infty$, we define the Graph Independent Subnetwork Training Kernel (\texttt{GIST-K}) as follows:
\begin{align*}
    \mathbf{G}^{\infty} = \bar{\mathbf{A}}\h^\infty\bar{\mathbf{A}}.
\end{align*}
Given the \texttt{GIST-K}, we adopt the following set of assumptions related to the underlying graph; see Appendix \ref{A:trans_input} for more details.

\begin{asump} \label{main_assm}
Assume $\lambda_{\min}(\mathbf{\bar{A}}) \neq 0$ and there exists $\epsilon\in(0,1)$ and $p\in\mathbb{Z}_+$ such that $(1-\epsilon)^2p\leq \mathbf{D}_{ii}\leq (1+\epsilon)^2p$ for all $i \in [n] = \{1, 2, \dots, n\}$, where $\mathbf{D}$ is the degree matrix.
Additionally, assume that $i)$ input node representations are bounded in norm and not parallel to any other node representation, $ii)$ output node representations are upper bounded, $iii)$ sub-GCN feature partitions are generated at each iteration from a categorical distribution with uniform mean $\frac{1}{m}$.
\end{asump}

Given this set of assumptions, we derive the following result
\begin{theorem} \label{main_gist_conv}
Given assumption \ref{main_assm}, if the number of hidden neurons within the two-layer GCN satisfies $d_1 = \Omega\left(\frac{n^3\zeta^2T^2}{\delta^2\gamma(1-\gamma)^2\lambda_0^4}\left(n + \frac{d}{m^2}\|\bar{\mathbf{A}}^2\|_{1,1}\right)\right)$, then \texttt{GIST} with step-size $\eta = O\left(\frac{\lambda_0}{n^2\|\mathbf{A}^2\|_{1,1}}\right)$ converges with probability $1 - \delta$ according to
\begin{small}
\begin{align*}
    &\E_{[\mathcal{M}_{t-1}],\boldsymbol{\Theta}_0,\mathbf{a}} \left[\left\|\y - \hat{\y}(t)\right\|_2^2\right] \\
    &\leq
    \left(\gamma + (1-\gamma)\left(1 - \tfrac{\eta\lambda_0}{2}\right)^\zeta\right)^t\E_{\boldsymbol{\Theta}_0,\mathbf{a}}\left[\|\y - \hat{\y}(0)\|_2^2\right] \\
    &\quad \quad \quad + O\left(\frac{(m-1)^2\zeta\|\bar{\mathbf{A}}^2\|_{1,1} n d}{\gamma m^2d_1}\right).
\end{align*}
\end{small}
\end{theorem}

\begin{table*}[!t]
\centering
\begin{footnotesize}
\setlength{\tabcolsep}{5pt}
\begin{tabular}{c|ccc|cccc}
\toprule
    $m$ & $d_0$ & $d_1$ & $d_2$ & Cora & Citeseer & Pubmed & OGBN-Arxiv\\ \midrule
    Baseline & & & & $81.52 \pm 0.005$ & $75.02 \pm 0.018$ & $75.90 \pm 0.003$ & $70.85 \pm 0.089$\\ \midrule
    2 & $\checkmark$ & $\checkmark$ & $\checkmark$ & $80.00 \pm 0.010$ & $\mathbf{75.95} \pm 0.007$ & $76.68 \pm 0.011$ & $65.65 \pm 0.700$\\
     & $\checkmark$ & $\checkmark$ &  & $78.30 \pm 0.011$ & $69.34 \pm 0.018$ & $75.78 \pm 0.015$ & $65.33 \pm 0.347$\\
     & & $\checkmark$ & $\checkmark$ & $\textbf{80.82} \pm 0.010$ & $75.82 \pm 0.008$ & $\textbf{78.02} \pm 0.007$ & $\textbf{70.10} \pm 0.224$ \\
     \midrule
    4 & $\checkmark$ & $\checkmark$ & $\checkmark$ & $76.78 \pm 0.017$ & $70.66 \pm 0.011$ & $65.67 \pm 0.044$ & $54.21 \pm 1.360$ \\
     & $\checkmark$ & $\checkmark$ &  & $66.56 \pm 0.061$  & $68.38 \pm 0.018$ & $68.44 \pm 0.014$ & $52.64 \pm 1.988$ \\
     & & $\checkmark$ & $\checkmark$ & $\mathbf{81.18} \pm 0.007$ & $\mathbf{76.21} \pm 0.017$ & $\mathbf{76.99} \pm 0.006$ & $\textbf{68.69} \pm 0.579$\\
     \midrule
    8 & $\checkmark$ & $\checkmark$ & $\checkmark$ & $48.32 \pm 0.087$ & $45.42 \pm 0.092$ & $54.29 \pm 0.029$ & $40.26 \pm 1.960$\\
     & $\checkmark$ & $\checkmark$ &  & $53.60 \pm 0.020$ & $54.68 \pm 0.030$ & $51.44 \pm 0.002$ & $26.84 \pm 7.226$\\
     & & $\checkmark$ &$\checkmark$ & $\mathbf{79.58} \pm 0.006$ & $\mathbf{75.39} \pm 0.016$ & $\mathbf{76.99} \pm 0.006$ & $\textbf{65.81} \pm 0.378$\\
 \bottomrule
\end{tabular}
\caption{Test accuracy of GCN models trained on small-scale datasets with \texttt{GIST}. We selectively partition each feature dimension within the GCN model, indicated by a check mark.
\textit{Partitioning on all hidden layers except the input layer leads to optimal performance.}}
\label{layer_split_results}
\end{footnotesize}
\end{table*}

A full proof of this result is deferred to Appendix \ref{A:theory}, but a sketch of the techniques used is as follows: \vspace{-0.2cm}
\begin{enumerate}[leftmargin=*]
    \item We define the \texttt{GIST-K} and show that it remains positive definite throughout training given our assumptions and sufficient overparameterization.
    \item \vspace{-0.25cm} We show that local sub-GCN training converges linearly, given a positive definite \texttt{GIST-K}.
    % \item Using the positive definiteness of the GISTK, we show that local training of the subnetworks are guaranteed to achieve a linear convergence rate
    \item \vspace{-0.25cm}%Following \cite{liao2021convergence},
    We analyze the change in training error when sub-GCNs are sampled (\texttt{subGCNs}), locally trained (\texttt{subTrain}), and aggregated (\texttt{subAgg}).
    \item \vspace{-0.25cm} We establish a connection between local and aggregated weight perturbation, showing that network parameters are bounded by a small region centered around the initialization given sufficient overparameterization.
\end{enumerate}

\vspace{-0.35cm}
\textbf{Discussion.}
Stated intuitively, the result in Theorem \ref{main_gist_conv} shows that, given sufficient width, two-layer GCNs trained using $\texttt{GIST}$ converge to approximately zero training error.
The convergence rate is linear and on par with training the full, two-layer GCN model, up to an error neighborhood (i.e., without the feature partition utilized in \texttt{GIST}).
Such theory shows that the feature partitioning strategy of \texttt{GIST} does not cause the model to diverge in training.
Additionally, the theory suggests that wider GCN models and a larger number of sub-GCNs should be used to maximize the convergence rate of \texttt{GIST} and minimize the impact of the additive term within Theorem \ref{main_gist_conv}; though the affect of $m$ on the radius is less significant compared to $d_1$.
Such findings reflect practical observations that are made within Section \ref{S:experiment} and reveal that \texttt{GIST} is particularly-suited towards training extremely wide models that cannot be trained using a traditional, centralized approach on a single GPU. 
% \begin{align*}
%      \h^\infty_{ij} = \frac{1}{d_1m}\inner{\mathbf{\hat{x}}_i}{\mathbf{\hat{x}}_j}\E_{\boldsymbol{\theta}\sim\mathcal{N}(0,\mathbf{I})}\left[\I\{\inner{\mathbf{\hat{x}}_i}{\boldsymbol{\theta}}\geq 0, \inner{\mathbf{\hat{x}}_j}{\boldsymbol{\theta}}\geq 0\}\right]
% \end{align*}
% define the Graph Independent Subnetwork Training Kernel (\texttt{GIST-k}) based on the scaled Gram matrix of

\section{Experiments} \label{S:experiment} 

We use \texttt{GIST} to train different GCN architectures on six public, multi-node classification datasets; see Appendix \ref{A:exp_det} for details.
In most cases, we compare the performance of models trained with \texttt{GIST} to that of models trained with standard methods (i.e., single GPU with node partitioning).
Comparisons to models trained with other distributed methodologies are also provided in Appendix \ref{A:other_dist}.
Experiments are divided into small and large scale regimes based upon graph size. 
% Small-scale experiments are used to run low-cost ablation experiments (i.e., see Section \ref{S:small_scale}).
% We also compare \texttt{GIST} to other distributed training methodologies in Section \ref{S:other_dist}.
% Cora, Citeseer, Pubmed, and OGBN-Arxiv \cite{sen2008collective, ogbn} are considered ``small-scale'' datasets and are used to run low-cost ablation experiments; see Section~\ref{S:small_scale}.
% Reddit \cite{graphsage} and Amazon2M \cite{clustergcn} are considered ``large-scale'' datasets.
% F1 score and training time are used to measure the performance of models trained with \texttt{GIST} on both of these datasets; see Section~\ref{S:large_scale}.
The goal of \texttt{GIST} is to $i)$ train GCN models to state-of-the-art performance, $ii)$ minimize wall-clock training time, and $iii)$ enable training of very wide GCN models.

% \vspace{-0.4cm}
\subsection{Small-Scale Experiments}
\label{S:small_scale}

In this section, we perform experiments over Cora, Citeseer, Pubmed, and OGBN-Arxiv datasets~\cite{sen2008collective, ogbn}.
For these small-scale datasets, we train a three-layer, 256-dimensional GCN model \cite{gcn} with \texttt{GIST}; see Appendix \ref{A:small_scale} for further experimental settings.
All reported metrics are averaged across five separate trials.
Because these experiments run quickly, we use them to analyze the impact of different design and hyperparameter choices rather than attempting to improve runtime (i.e., speeding up such short experiments is futile).

\textbf{Which layers should be partitioned?}
We investigate whether models trained with \texttt{GIST} are sensitive to the partitioning of features within certain layers.
Although the output dimension $d_3$ is never partitioned, we selectively partition dimensions $d_0$, $d_1$, and $d_2$ to observe the impact on model performance; see Table~\ref{layer_split_results}.
Partitioning input features ($d_0$) significantly degrades test accuracy because sub-GCNs observe only a portion of each node's input features (i.e., this becomes more noticeable with larger $m$).
% For example, partitioning $d_0$ decreases model test accuracy from $79.58\%$ to $48.32\%$ on Cora when $m=8$.
% Intuitively, this performance decrease occurs because each sub-GCN observes only a portion of node input features (i.e., each sub-GCN has $d_0/m$-dimensional input).
However, other feature dimensions cause no performance deterioration when partitioned between sub-GCNs, leading us to partition all feature dimensions other than $d_0$ and $d_L$ within the final \texttt{GIST} methodology; see Figure~\ref{fig:main}-(b). 
% Aside from the input layer, all GCN \emph{hidden} layers (i.e., dimensions $d_1$ and $d_2$ in this case) can be partitioned without degrading model performance.
% Therefore, \texttt{GIST} adopts the strategy of partitioning all dimensions other than $d_0$ and $d_L$ (i.e., input and output dimensions); see Figure~\ref{fig:main}-(b) for a depiction of this strategy.

\begin{table*}[!t]
\centering
\begin{footnotesize}
\begin{tabular}{cccccccccccccc}
\toprule
\multirow{3}{*}{$L$} & \multirow{3}{*}{$m$}& \multicolumn{6}{c}{Reddit Dataset} &\multicolumn{6}{c}{Amazon2M Dataset}\\
& & \multicolumn{3}{c}{GraphSAGE} & \multicolumn{3}{c}{GAT} & \multicolumn{3}{c}{GraphSAGE ($d_i = 400$)} & \multicolumn{3}{c}{GraphSAGE ($d_i = 4096$)}\\
\cmidrule{3-14}
&& F1 & Time & Speedup & F1 & Time & Speedup & F1 & Time & Speedup & F1 & Time & Speedup\\
\midrule
2 & - & 96.09 & 105.78s & $1.00\times$ & 89.57 & 1.19hr & $1.00\times$ & 89.90 & 1.81hr & $1.00\times$ & 91.25 & 5.17hr & $1.00\times$ \\
& 2 & 96.40 & 70.29s & $1.50\times$& 90.28 & 0.58hr & $2.05\times$ & 88.36 & 1.25hr & ($1.45\times$) & 90.70 & 1.70hr & $3.05\times$ \\
& 4 & 96.16 & 68.88s & $1.54\times$ & 90.02 & 0.31hr & $3.86\times$ & 86.33 & 1.11hr & ($1.63\times$) & 89.49 & 1.13hr & ($4.57\times$)\\
& 8 & 95.46 & 76.68s & $1.38\times$ & 89.01 & 0.18hr & $6.70\times$ & 84.73 & 1.13hr & ($1.61\times$) & 88.86 & 1.11hr & ($4.65\times$)\\
\midrule
3 & - & 96.32 & 118.37s & $1.00\times$ & 89.25 & 2.01hr & $1.00\times$ & 90.36 & 2.32hr & $1.00\times$ & 91.51 & 9.52hr & $1.00\times$\\
& 2 & 96.36 & 80.46s & $1.47\times$ & 89.63 & 0.95hr & $2.11\times$ & 88.59 & 1.56hr & ($1.49\times$) & 91.12 & 2.12hr & $4.49\times$ \\
& 4 & 95.76 & 78.74s & $1.50\times$ & 88.82 & 0.48hr & $4.19\times$ & 86.46 & 1.37hr & ($1.70\times$) & 89.21 & 1.42hr & ($6.72\times$)\\
& 8 & 94.39 & 88.54s & ($1.34\times$) & 70.38 & 0.26hr & ($7.67\times$) & 84.76 & 1.37hr & ($1.69\times$) & 86.97 & 1.34hr & ($7.12\times$)\\
\midrule
4 & - & 96.32 & 120.74s & $1.00\times$ & 88.36 & 2.77hr & $1.00\times$ & 90.40 & 3.00hr &  $1.00\times$ & 91.61 & 14.20hr &  $1.00\times$ \\
& 2 & 96.01 & 91.75s & $1.32\times$ & 87.97 & 1.31hr & $2.11\times$ & 88.56 & 1.79hr & ($1.68\times$) & 91.02 & 2.77hr & $5.13\times$ \\
& 4 & 95.21  & 78.74s & ($1.53\times$) & 78.42 & 0.66hr & ($4.21\times$) & 87.53 & 1.58hr & ($1.90\times$) & 89.07 & 1.65hr & ($8.58\times$)\\
& 8 & 92.75 & 88.71s & ($1.36\times$) & 66.30 & 0.35hr & ($7.90\times$) & 85.32 & 1.56hr & ($1.93\times$) & 87.53 & 1.55hr & ($9.13\times$)\\
\bottomrule
\end{tabular}
\caption{Performance of models trained with \texttt{GIST} on Reddit and Amazon2M.
Parenthesis are placed around speedups achieved at a cost of $>$1 deterioration in F1 and $m=$``-'' refers to the baseline.
\textit{Models trained with \texttt{GIST} train more quickly and achieve comparable F1 score to those trained with standard methodology. The performance benefits of \texttt{GIST} become more pronounced for wider models.}}
\label{tab:reddit_am2m}
\end{footnotesize}
\end{table*}

\textbf{How many Sub-GCNs to use?}
Using more sub-GCNs during \texttt{GIST} training typically improves runtime because sub-GCNs $i)$ become smaller, $ii)$ are each trained for fewer epochs, and $iii)$ are trained in parallel.
We find that all models trained with \texttt{GIST} perform similarly for practical settings of $m$; see Table~\ref{layer_split_results}.
% In fact, model performance often improves as the number of sub-GCNs is increased.
One may continue increasing the number sub-GCNs used within \texttt{GIST} until all GPUs are occupied or model performance begins to decrease.
% The latter effect is more noticeable in large-scale experiments; see Tables~\ref{tab:reddit_am2m} and~\ref{tab:reddit_am2m}.

\textbf{\texttt{GIST} Performance.}
Models trained with \texttt{GIST} often exceed the performance of models trained with standard, single-GPU methodology; see Table~\ref{layer_split_results}.
Intuitively, we hypothesize that the random feature partitioning within \texttt{GIST}, which loosely resembles dropout \cite{srivastava2014dropout}, provides regularization benefits during training, but we leave an in-depth analysis of this property as future work.
% However, we leave an in-depth analysis of the performance benefits derived from \texttt{GIST} as future work.

\subsection{Large-Scale Experiments}
\label{S:large_scale}
For large-scale experiments on Reddit and Amazon2M, the baseline model is trained on a single GPU and compared to models trained with \texttt{GIST} in terms of F1 score and training time.
All large-scale graphs are partitioned into $15,\!000$ sub-graphs during training.\footnote{Single-GPU training with graph partitioning via METIS is the same approach adopted by ClusterGCN \cite{clustergcn}, making our single-GPU baseline a ClusterGCN model. We adopt the same number of sub-graphs as proposed in this work.}
Graph partitioning is mandatory because the training graphs are too large to fit into memory.
One could instead use layer sampling to make training tractable (see Section \ref{S:gist_layer_samp}), but we adopt graph partitioning in most experiments because the implementation is simple and performs well.

\textbf{Reddit Dataset.}
We perform tests with $256$-dimensional GraphSAGE \cite{graphsage} and GAT \cite{gat} models with two to four layers on Reddit; see Appendix \ref{A:large_scale} for more details.
As shown in Table~\ref{tab:reddit_am2m}, utilizing \texttt{GIST} significantly accelerates GCN training (i.e., a $1.32\times$ to $7.90\times$ speedup).
\texttt{GIST} performs best in terms of F1 score with $m=2$ sub-GCNs (i.e., $m=4$ yields further speedups but F1 score decreases).
% Although $m=4$ reduces training time, the F1 score also decreases slightly.
Interestingly, the speedup provided by \texttt{GIST} is more significant for models and datasets with larger compute requirements.
For example, experiments with the GAT architecture, which is more computationally expensive than GraphSAGE, achieve a near-linear speedup with respect to $m$.

% For example, the speedup achieved on GAT experiments is more significant than the speedup achieved on GraphSAGE experiments, which becomes especially noticeable for larger $m$ (i.e., \texttt{GIST} provides a near-linear speedup with respect to $m$ when training the GAT model). 
% Additionally, larger speedups are observed when training GraphSAGE models with \texttt{GIST} on the Amazon2M dataset; see Table \ref{tab:reddit_am2m}.

\textbf{Amazon2M Dataset.}
Experiments are performed with two, three, and four-layer GraphSAGE models \cite{graphsage} with hidden dimensions of $400$ and $4096$ (we refer to these models as ``narrow'' and ``wide'', respectively).
We compare the performance (i.e., F1 score and wall-clock training time) of GCN models trained with standard, single-GPU methodology to that of models trained with \texttt{GIST}; see Table \ref{tab:reddit_am2m}.
Narrow models trained with \texttt{GIST} have a lower F1 score in comparison to the baseline, but training time is significantly reduced.
% For example, when $L=4$, the narrow GCN trained with \texttt{GIST} achieves an F1 score 1.84 below the baseline with $1.68\times$ lower runtime. 
For wider models, \texttt{GIST} provides a more significant speedup (i.e., up to $7.12\times$) and tends to achieve comparable F1 score in comparison to the baseline, revealing that \textit{\texttt{GIST} works best with wider models}.

Within Table~\ref{tab:reddit_am2m}, models trained with \texttt{GIST} tend to achieve a wall-clock speedup at the cost of a lower F1 score (i.e., observe the speedups marked with parenthesis in Table~\ref{tab:reddit_am2m}).
When training time is analyzed with respect to a fixed F1 score, we observe that the baseline takes significantly longer than \texttt{GIST} to achieve a fixed F1 score.
For example, when $L=2$, a wide GCN trained with \texttt{GIST} ($m=8$) reaches an F1 score of 88.86 in $\sim$4,000 seconds, \emph{while models trained with standard methodology take $\sim$10,000 seconds to achieve a comparable F1 score.}
As such, \texttt{GIST} significantly accelerates training relative to model performance.
% Such a finding reveals that, although \texttt{GIST} may achieve a lower final F1 score in comparison to the baseline, training is still significantly accelerated with respect to model performance. 

\begin{table*}[!t]
\begin{footnotesize}
\centering
\begin{tabular}{ccccccc}
\toprule
\multirow{2}{*}{$L$} & \multirow{2}{*}{$m$} & 
\multicolumn{5}{c}{F1 Score~(Time)} \\ 
\cmidrule(l){3-7} & 
 & \multicolumn{1}{c}{$d_i=400$} 
 & \multicolumn{1}{c}{$d_i=4096$} 
 & \multicolumn{1}{c}{$d_i=8192$} 
 & \multicolumn{1}{c}{$d_i=16384$} 
 & \multicolumn{1}{c}{$d_i=32768$}\\
\midrule
2 
& -
& 89.38~(1.81hr)
& 90.58~(5.17hr)
& OOM
& OOM
& OOM\\
& 2 
& 87.48~(1.25hr)
& 90.09~(1.70hr)  
& 90.87~(2.76hr)
& 90.94~(9.31hr)
& 90.91~(32.31hr)\\

& 4 
& 84.82~(1.11hr)
& 88.79~(1.13hr)  
& 89.76~(1.49hr)
& 90.10~(2.24hr)
& 90.17~(5.16hr)\\
& 8 
& 82.56~(1.13hr)
& 87.16~(1.11hr)  
& 88.31~(1.20hr)
& 88.89~(1.39hr)
& 89.46~(1.76hr)\\
\midrule
3 & -
& 89.73~(2.32hr)
& 90.99~(9.52hr) 
& OOM
& OOM
& OOM\\
& 2 
& 87.79~(1.56hr)
& 90.40~(2.12hr)  
& 90.91~(4.87hr)
& 91.05~(17.7hr)
& OOM\\
& 4 
& 85.30~(1.37hr)
& 88.51~(1.42hr) 
& 89.75~(2.07hr)
& 90.15~(3.44hr)
& OOM\\
& 8 
& 82.84~(1.37hr)
& 86.12~(1.34hr)  
& 88.38~(1.37hr)
& 88.67~(1.88hr)
& 88.66~(2.56hr)\\
\midrule
4 & - 
& 89.77~(3.00hr)
& 91.02~(14.20hr)  
& OOM
& OOM
& OOM\\
& 2 
& 87.75~(1.79hr)
& 90.36~(2.77hr)  
& 91.08~(6.92hr)
& 91.09~(26.44hr)
& OOM\\
& 4 
& 85.32~(1.58hr)
& 88.50~(1.65hr)  
& 89.76~(2.36hr)
& 90.05~(4.93hr)
& OOM\\
& 8 
& 83.45~(1.56hr)
& 86.60~(1.55hr) 
& 88.13~(1.61hr)
& 88.44~(2.30hr)
& OOM\\
\bottomrule
\end{tabular}
\caption{Performance of GraphSAGE models of different widths trained with \texttt{GIST} on Amazon2M.
$m=$``-'' refers to the baseline and ``OOM'' marks experiments that cause out-of-memory errors.
% Experiments marked with ``OOM'' cause an out-of-memory error during training.
\textit{\texttt{GIST} enables training of higher-performing, ultra-wide models.}}
\label{ultrawide}
\end{footnotesize}
\end{table*}

\subsection{Training Ultra-Wide GCNs} 
\label{S:ultra}
% To illustrate the power of our proposed methodology, we leverage \texttt{GIST} to train models of shocking scale (i.e., ``ultra-wide'' models) on the Amazon2M dataset.
\textbf{We use \texttt{GIST} to train GraphSAGE models with widths as high as 32K (i.e., $\boldsymbol{8\times}$ beyond the capacity of a single GPU)}; see Table \ref{ultrawide} for results and Appendix \ref{A:ultra} for more details.
% For $d_i > 4096$, evaluation must be performed on graph partitions (not the full graph) to avoid memory overflow.
% As such, the graph is partitioned into 5,000 sub-graphs during testing and F1 score is measured over each partition and averaged.
% The performance of ultra-wide models is reported in Table~\ref{ultrawide}.
Considering $L=2$, the best-performing, single-GPU GraphSAGE model ($d_i=4096$) achieves an F1 score of $90.58$ in $5.2$ hours.
% Without \texttt{GIST}, the best-performing two-layer GraphSAGE model is of dimension $d_i=4096$, which achieves an F1 score of $90.58$ in $5.2$ hours.
With \texttt{GIST} ($m=2$), we achieve a higher F1 score of $90.87$ in $2.8$ hours (i.e., a $1.86\times$ speedup) using $d_i=8192$, which is beyond single GPU capacity.
Similar patterns are observed for deeper models.
Furthermore, we find that utilizing larger hidden dimensions yields further performance improvements, revealing the utility of wide, overparameterized GCN models.
\emph{\texttt{GIST}, due to its feature partitioning strategy, is unique in its ability to train models of such scale to state-of-the-art performance}.
% The ability to train models of such scale to state-of-the-art performance affirms the ability of \texttt{GIST} to enable large-scale experiments on graphs. 

% A model of such width cannot be trained with standard methodology, and larger hidden dimensions yield further improvements, revealing that wide, overparameterized GCN models are useful for large-scale datasets such as Amazon2M.

% Similar patterns are observed for deeper models trained with \texttt{GIST}.
% For example, the four-layer, $8192$-dimensional GraphSAGE model trained with \texttt{GIST} ($m=2$) achieves an F1 score of 91.08 in 7 hours---achieving similar F1 score with a $4196$-dimensional model trained with standard methodology takes over 14 hours.
% The ability to train models of such scale to state-of-the-art performance affirms the ability of \texttt{GIST} to enable large-scale experiments on graphs. 

\subsection{\texttt{GIST} with Layer Sampling} \label{S:gist_layer_samp}
As previously mentioned, some node partitioning approach must be adopted to avoid memory overflow when the underlying training graph is large.
Although graph partitioning is used within most experiments (see Section \ref{S:large_scale}), \texttt{GIST} is also compatible with other node partitioning strategies.
To demonstrate this, we perform training on Reddit using \texttt{GIST} combined with a recent layer sampling approach \cite{ladies} (i.e., instead of graph partitioning); see Appendix \ref{A:gist_layer_samp} for more details.

As shown in Table \ref{ladies_reddit}, combining \texttt{GIST} with layer sampling enables training on large-scale graphs, and the observed speedup actually exceeds that of \texttt{GIST} with graph partitioning.
% Additionally, the speedup observed in combining \texttt{GIST} with layer sampling actually exceeds the speedup when graph partitioning is used.
For example, \texttt{GIST} with layer sampling yields a $1.83\times$ speedup when $L=2$ and $m=2$, in comparison to a $1.50\times$ speedup when graph partitioning is used within \texttt{GIST} (see Table \ref{tab:reddit_am2m}).
As the number of sub-GCNs is increased beyond $m=2$, \texttt{GIST} with layer sampling continues to achieve improvements in wall-clock training time (e.g., speedup increases from $1.83\times$ to $2.90\times$ from $m=2$ to $m=4$ for $L=2$) without significant deterioration to model performance.
%, whereas training with \texttt{GIST} alone tends to reach a plateau in wall-clock time as $m$ is increased; see Table \ref{tab:reddit_am2m}
% Additionally, these speedups with increased $m$ are observed without significant deterioration to model performance (i.e., F1 score is very similar for $m=2$ and $m=4$ in Table \ref{ladies_reddit}).
Thus, although node partitioning is needed to enable training on large-scale graphs, the feature partitioning strategy of \texttt{GIST} is compatible with numerous sampling strategies (i.e., not just graph sampling).
% Such findings demonstrate that \texttt{GIST} is compatible with numerous node sampling techniques, which are mandatory components of GCN training for large-scale graphs.
% Such findings demonstrate the benefits of combining \texttt{GIST} with compatible layer sampling techniques. 

\begin{table}[!t]
\centering
\begin{footnotesize}
\begin{tabular}{ccccc}
\toprule
\multirow{2}{*}{$L$} & \multirow{2}{*}{\# Sub-GCNs} & \multicolumn{3}{c}{\texttt{GIST} + LADIES} \\
\cmidrule{3-5}
&& F1 Score & Time & Speedup \\
\midrule
2 & Baseline & 89.73 & 3359.91s & $1.00\times$\\
& 2 & 89.29 & 1834.59s & $1.83\times$ \\
& 4 & 88.42 & 1158.51s & $2.90\times$ \\
%& 8 & 84.70 & 792.58s & $4.24\times$ \\
\midrule
3 & Baseline & 89.57 & 4803.88s & $1.00\times$ \\
& 2 & 86.52 & 2635.18s & $1.82\times$ \\
& 4 & 86.72 & 1605.32s & $3.00\times$ \\
%& 8 & 80.57 & 1060.74s & $4.53\times$ \\
%\midrule
%4 & Baseline & 89.89 & 6286.57s & $1.00\times$ \\
%& 2 & 82.23 & 3399.79s & $1.85\times$  \\
%& 4 & 83.44  & 2050.01s & $3.67\times$ \\
%& 8 & 76.76 & 1307.36s & $4.81\times$  \\
\bottomrule
\end{tabular}
\caption{Performance of GCN models trained with a combination of \texttt{GIST} and LADIES \cite{ladies} on Reddit. Here, the baseline represents models trained with LADIES in a standard, single-GPU manner. \emph{Combining \texttt{GIST} with layer sampling leads to further improvements in wall-clock training time without deteriorating the F1 score.}}
\label{ladies_reddit}
\end{footnotesize}
\end{table}

\section{Conclusion}
We present~\texttt{GIST}, a distributed training approach for GCNs that enables the exploration of larger models and datasets.
\texttt{GIST} is compatible with existing sampling approaches and leverages a feature-wise partition of model parameters to construct smaller sub-GCNs that are trained independently and in parallel.
We have shown that \texttt{GIST} achieves remarkable speed-ups over large graph datasets and even enables the training of GCN models of unprecedented size.
We hope \texttt{GIST} can empower the exploration of larger, more powerful GCN architectures within the graph community.

\newpage
\bibliography{example_paper}
\bibliographystyle{icml2022}

\newpage
\onecolumn
\appendix
\section{Experimental Details} \label{A:exp_det}
\subsection{Datasets} \label{A:dataset}
The details of the datasets utilized within \texttt{GIST} experiments in Section \ref{S:experiment} are provided in Table \ref{tab:datasets}.
Cora, Citeseer, PubMed and OGBN-Arxiv are considered ``small-scale'' datasets and are utilized within experiments in Section \ref{S:small_scale}.
Reddit and Amazon2M are considered ``large-scale'' datasets and are utilized within experiments in Section \ref{S:large_scale}.

\begin{table}[!h]
    \centering
    \begin{footnotesize}
    \begin{tabular}{l|cccc}
        \toprule
        Dataset & $n$ & \# Edges & \# Labels & $d$\\ 
        \midrule
        Cora & 2,708 & 5,429 & 7 & 1,433\\ 
        CiteSeer & 3,312 & 4,723 & 6 & 3,703 \\ 
        Pubmed & 19,717 & 44,338 & 3 & 500 \\ 
        OGBN-Arxiv & 169,343 & 1.2M & 40 & 128 \\ 
        Reddit & 232,965 & 11.6 M & 41 & 602 \\ 
        Amazon2M & 2.5 M & 61.8 M & 47 & 100 \\
        \bottomrule
    \end{tabular}
    \caption{Details of relevant datasets.}
    \label{tab:datasets}
    \end{footnotesize}
\end{table}

\subsection{Implementation Details} \label{A:implementation}
We provide an implementation of \texttt{GIST} in PyTorch~\cite{pytorch} using the NCCL distributed communication package for training GCN \cite{gcn}, GraphSAGE~\cite{graphsage} and GAT \cite{gat} architectures.
Our implementation is centralized, meaning that a single process serves as a central parameter server.
From this central process, the weights of the global model are maintained and partitioned to different worker processes (including itself) for independent training.
Experiments are conducted with $8$ NVIDIA Tesla V100-PCIE-32G GPUs, a 56-core Intel(R) Xeon(R) CPU E5-2680 v4 @ 2.40GHz, and 256~GB of RAM.
% \onecolumn

\subsection{Small-Scale Experiments} \label{A:small_scale}
Small-scale experiments in Section \ref{S:small_scale} are performed using Cora, Citeseer, Pubmed, and OGBN-Arxiv datasets~\cite{sen2008collective, ogbn}.
\texttt{GIST} experiments are performed with two, four, and eight sub-GCNs in all cases.
We find that the performance of models trained with \texttt{GIST} is relatively robust to the number of local iterations $\zeta$, but test accuracy decreases slightly as $\zeta$ increases; see Figure \ref{local_iter_fig}.
Based on the results in Figure \ref{local_iter_fig}, we adopt $\zeta=20$ for Cora, Citeseer, and Pubmed, as well as $\zeta=100$ for OGBN-Arxiv. 

\begin{figure*}[!h]
    \centering
    \includegraphics[width=\linewidth]{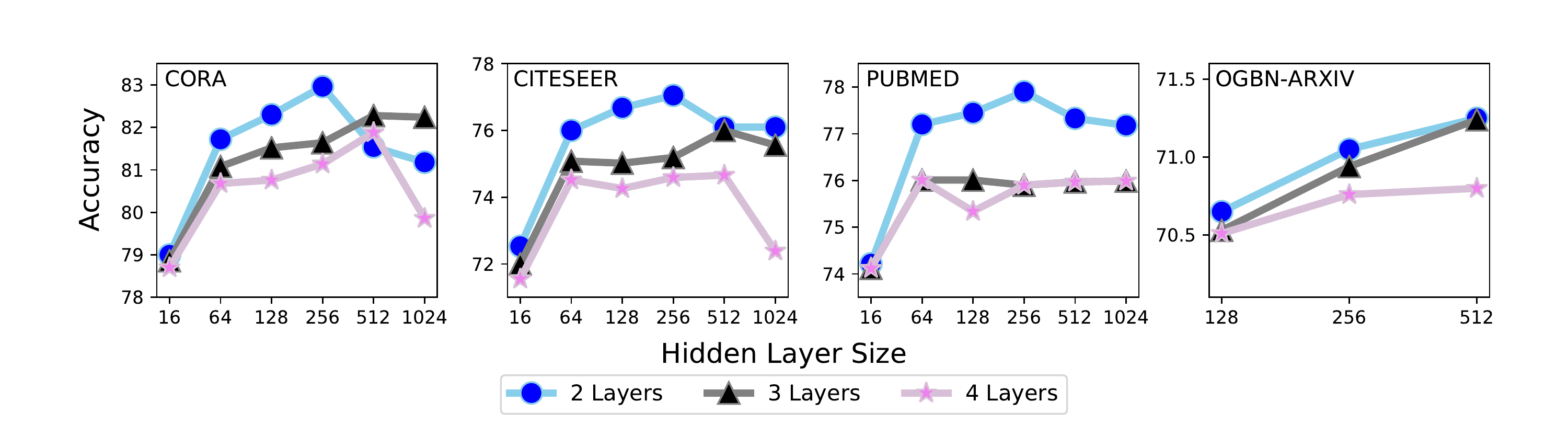}
    \caption{Test accuracy for different sizes (i.e., varying depth and width) of GCN models trained with standard, single-GPU methodology on small-scale datasets. \emph{We adopt three-layer, 256-dimensional GCN models as our baseline architecture.}}
    \label{single_gpu_fig}
\end{figure*}

\begin{figure*}[!h]
    \centering
    \includegraphics[width=\linewidth]{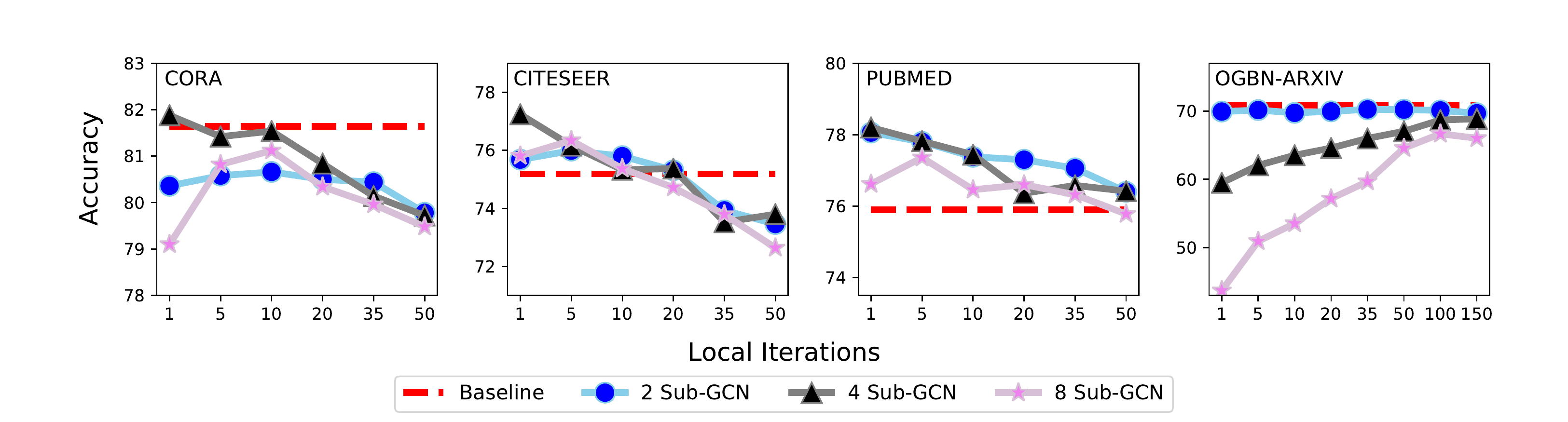}
    \caption{Test accuracy of GCN models trained on small-scale datasets with \texttt{GIST} using different numbers of local iterations and sub-GCNs. \textit{Models trained with \texttt{GIST} are surprisingly robust to the number of local iterations used during training, no matter the number of sub-GCNs.}}
    \label{local_iter_fig}
\end{figure*}

% \textbf{Incorporating Local Iterations.}
% \texttt{GIST} dictates that sub-GCNs be trained for $\zeta$ independent, local iterations between synchronization rounds~\cite{use_local_sgd}.
% Although increasing $\zeta$ reduces communication for a fixed number of epochs, this may come at the cost of degraded model performance.
% We train models using \texttt{GIST} with different settings for $\zeta$; see Figure~\ref{local_iter_fig}.
% The performance of \texttt{GIST} is relatively robust to the number of local iterations, but test accuracy decreases slightly as $\zeta$ increases.
% For small-scale datasets, using $\zeta=20$ performs consistently well, but $\zeta$ can be further increased (e.g., $\zeta = 500$ or $\zeta = 5000$) on large-scale datasets without noticeable performance deterioration; see Section~\ref{S:large_scale}.
% Additionally, larger values of $\zeta$, such as $\zeta=100$, seem to perform best on OGBN-Arxiv.

Experiments are run for 400 epochs with a step learning rate schedule (i.e., $10\times$ decay at 50\% and 75\% of total epochs). 
A vanilla GCN model, as described in~\cite{gcn}, is used.
The model is trained in a full-batch manner using the Adam optimizer~\cite{adam}.
No node sampling techniques are employed because the graph is small enough to fit into memory. 
All reported results are averaged across five trials with different random seeds.
For all models, $d_0$ and $d_L$ are respectively given by the number of features and output classes in the dataset.
The size of all hidden layers is the same, but may vary across experiments.

We first train baseline GCN models of different depths and hidden dimensions using a single GPU to determine the best model depth and hidden dimension to be used in small-scale experiments.
The results are shown in Figure~\ref{single_gpu_fig}.
Deeper models do not yield performance improvements for small-scale datasets, but test accuracy improves as the model becomes wider.
Based upon the results in Figure~\ref{single_gpu_fig}, we adopt a three-layer GCN with a hidden dimension of $d_1 \!= \!d_2 \!= \!256$ as the underlying model used in small-scale experiments.
Though two-layer models seem to perform best, we use a three-layer model within Section \ref{S:small_scale} to enable more flexibility in examining the partitioning strategy of \texttt{GIST}.

\subsection{Large-Scale Experiments} \label{A:large_scale}
\noindent \textbf{Reddit Dataset.}
For experiments on Reddit, we train 256-dimensional GraphSAGE and GAT models using both \texttt{GIST} and standard, single-GPU methodology.
During training, the graph is partitioned into $15,\!000$ sub-graphs.
Training would be impossible without such partitioning because the graph is too large to fit into memory.
The setting for the number of sub-graphs is the optimal setting proposed in previous work \cite{clustergcn}.
Models trained using \texttt{GIST} and standard, single-GPU methodologies are compared in terms of F1 score and training time.

All tests are run for $80$ epochs with no weight decay, using the Adam optimizer \cite{adam}.
We find that $\zeta=500$ achieves consistently high performance for models trained with \texttt{GIST} on Reddit.
We adopt a batch size of 10 sub-graphs throughout the training process, which is the optimal setting proposed in previous work \cite{clustergcn}.

\medskip
\noindent \textbf{Amazon2M Dataset.}
For experiments on Amazon2M, we train two to four layer GraphSAGE models with hidden dimensions of $400$ and $4096$ using both \texttt{GIST} and standard, single-GPU methodology.
We follow the experimental settings of \cite{clustergcn}.
The training graph is partitioned into $15,\!000$ sub-graphs and a batch size of $10$ sub-graphs is used.
We find that using $\zeta = 5000$ performs consistently well.
Models are trained for $400$ total epochs with the Adam optimizer~\cite{adam} and no weight decay.

\subsection{Training Ultra-Wide GCNs} \label{A:ultra}
All settings for ultra-wide GCN experiments in Section \ref{S:ultra} are adopted from the experimental settings of Section \ref{S:large_scale}; see Appendix \ref{A:large_scale} for further details.
For $d_i > 4096$ evaluation must be performed on graph partitions (not the full graph) to avoid memory overflow.
As such, the graph is partitioned into $5,\!000$ sub-graphs during testing and F1 score is measured over each partition and averaged.
All experiments are performed using a GraphSAGE model, and the hidden dimension of the underlying model is changed between different experiments.

\subsection{\texttt{GIST} with Layer Sampling} \label{A:gist_layer_samp} 
Experiments in Section \ref{S:gist_layer_samp} adopt the same experimental settings as Section \ref{S:large_scale} for the Reddit dataset; see Appendix \ref{A:large_scale} for further details. 
Within these experiments, we combine \texttt{GIST} with LADIES \cite{ladies}, a recent layer sampling approach for efficient GCN training.
LADIES is used instead of graph partitioning.
Any node sampling approach can be adopted---some sampling approach is just needed to avoid memory overflow.

We train 256-dimensional GCN models with either two or three layers.
We utilize a vanilla GCN model within this section (as opposed to GraphSAGE or GAT) to simplify the implementation of \texttt{GIST} with LADIES, which creates a disparity in F1 score between the results in Section \ref{S:gist_layer_samp} and Section \ref{S:large_scale}.
Experiments in Section \ref{S:gist_layer_samp} compare the performance of the same models trained either with \texttt{GIST} or using standard, single-GPU methodology.
In this case, the single-GPU model is just a GCN trained with LADIES.

% Adopting the same experimental settings as Section \ref{S:large_scale}, we now combine \texttt{GIST} with LADIES \cite{ladies}, a recent layer sampling approach for GCN training, to train 256-dimensional GCN models\footnote{Note that we utilize a vanilla GCN model within this section (as opposed to GraphSAGE or GAT in Section \ref{S:large_scale}) in order to simplify the implementation of \texttt{GIST} with LADIES, which creates a disparity in F1 score between the two sections.} of varying depths on the Reddit dataset; see Table \ref{ladies_reddit}.

%%%%%%%%%%%%%%%%%%%%%%%%%%%%%%%%%%%%%%%%%%%%%%%%%%%%%%%%%%%%%%%%%%%%%%%%%%%%%%%
%%%%%%%%%%%%%%%%%%%%%%%%%%%%%%%%%%%%%%%%%%%%%%%%%%%%%%%%%%%%%%%%%%%%%%%%%%%%%%%

\section{Comparisons to Other Distributed Training Methodologies} \label{A:other_dist}
Although \texttt{GIST} has been shown to provide benefits in terms of GCN performance and training efficiency in comparison to standard, single-GPU training, other choices for the distributed training of GCNs exist.
Within this section, we compare \texttt{GIST} to other natural choices for distributed training, revealing that GCN models trained with \texttt{GIST} achieve favorable performance in comparison to those trained with other common distributed training techniques. 

\begin{table}
\centering
\begin{footnotesize}
\begin{tabular}{cc|cc}
\toprule
\# Machines & Method & F1 Score & Training Time\\
\midrule
2 & Local SGD & 96.37 & 137.17s\\
& \texttt{GIST} & \textbf{96.40} & \textbf{108.67s}\\
\midrule
4 & Local SGD & 95.00 & 127.63s\\
& \texttt{GIST} & \textbf{96.16} & \textbf{116.56s}\\
\midrule
8 & Local SGD & 93.40 & 129.58s\\
& \texttt{GIST} & \textbf{95.46} & \textbf{123.83s}\\
\bottomrule
\end{tabular}
\caption{Performance of GraphSAGE models trained using local SGD and \texttt{GIST} on Reddit. We adopt settings described in Section \ref{S:large_scale}, but use 100 local iterations for both \texttt{GIST} and local SGD. \emph{Models trained with \texttt{GIST} outperform those trained with local SGD in terms of test F1 score and wall-clock training time in all cases.}}
\label{lsgd_results}
\end{footnotesize}
\end{table}

\subsection{Local SGD}
% \noindent \textbf{Local SGD.}
A simple version of local SGD \cite{use_local_sgd} can be implemented for distributed training of GCNs by training the full model on each separate worker for a certain number of local iterations and intermittently averaging local updates.
In comparison to such a methodology, \texttt{GIST} has better computational and communication efficiency because $i)$ it communicates only a small fraction of model parameters to each machine and $ii)$ locally training narrow sub-GCNs is faster than locally training the full model.
We perform a direct comparison between local SGD and \texttt{GIST} on the Reddit dataset using a two-layer, 256-dimensional GraphSAGE model; see Table \ref{lsgd_results}.
As can be seen, GCN models trained with \texttt{GIST} have lower wall-clock training time and achieve better performance than those trained with local SGD in all cases.

\begin{table}[!htp]
\centering
\begin{footnotesize}
\begin{tabular}{cc|cc}
\toprule
\# Machines & Method & F1 Score & Inference Time\\
\midrule
2 & Ensemble & 96.31 & 3.59s\\
& \texttt{GIST} & \textbf{96.40} & \textbf{1.81s}\\
\midrule
4 & Ensemble & 96.10 & 6.38s\\
& \texttt{GIST} & \textbf{96.16} & \textbf{1.81s}\\
\midrule
8 & Ensemble & 95.28 & 11.95s\\
& \texttt{GIST} & \textbf{95.46} & \textbf{1.81s}\\
\bottomrule
\end{tabular}
\caption{Performance of GraphSAGE models trained both with \texttt{GIST} and as ensembles of shallow sub-GCNs on Reddit. \emph{Models trained with \texttt{GIST} perform better and do not suffer from increased inference time as the number of sub-GCNs is increased}.}
\label{ensemble_results}
\end{footnotesize}
\end{table}

% \noindent \textbf{Sub-GCN Ensembles.}
\subsection{Sub-GCN Ensembles}
As previously mentioned, increasing the number of local iterations (i.e., $\zeta$ in Algorithm \ref{alg:gist}) decreases communication requirements given a fixed amount of training.
When taken to the extreme (i.e., $\zeta \rightarrow \infty$), one could minimize communication requirements by never aggregating sub-GCN parameters, thus forming an ensemble of independently-trained sub-GCNs.
We compare \texttt{GIST} to such a methodology\footnote{For each sub-GCN, we measure validation accuracy throughout training and add the highest-performing model into the ensemble.} in Table \ref{ensemble_results} using a two-layer, 256-dimensional GraphSAGE model on the Reddit dataset.
Though training ensembles of sub-GCNs minimizes communication, Table \ref{ensemble_results} reveals that $i)$ models trained with \texttt{GIST} achieve better performance and $ii)$ inference time for sub-GCN ensembles becomes burdensome as the number of sub-GCNs is increased. 

\section{Theoretical Results} \label{A:theory}
\subsection{Formulation of \texttt{GIST} for One-Hidden-Layer GCNs} \label{A:theory_prelim}
In our analysis, we consider a GCN with one hidden-layer and a ReLU activation.
We assume that the GCN outputs a scalar value $\tilde{y}_i$ for each node in the graph.
Denoting $\tilde{\y} = [\tilde{y}_1,\dots,\tilde{y}_n]$, we can write the output of the GCN as
\begin{align*}
    \tilde{\y} = \frac{1}{\sqrt{d_1}}\bar{\mathbf{A}}\sigma(\bar{\mathbf{A}}\mathbf{X}\boldsymbol{\Theta})\mathbf{a}
\end{align*}
where $\boldsymbol{\Theta} = [\boldsymbol{\theta}_1,\dots,\boldsymbol{\theta}_{d_1}]\in\R^{n\times d_1}$ is the weights within the GCN's first layer and $\mathbf{a} = [a_1,\dots, a_{d_1}]\in \R^{d_1}$ is the weights within the GCN's second layer.
To simplify the analysis, we denote $\hat{\mathbf{X}} = \bar{\mathbf{A}}\mathbf{X} = [\mathbf{\hat{x}}_1,\dots\mathbf{\hat{x}}_n]$.
Then, we have $\hat{\mathbf{x}}_i = \sum_{i'=1}^n\bar{\mathbf{A}}_{ii'}\mathbf{x}_{i'}$ and the output of each node within the graph can be written as
\begin{align*}
    \tilde{y}_i = \frac{1}{\sqrt{d_1}}\sum_{i'=1}^n\sum_{r=1}^{d_1}\bar{\mathbf{A}}_{ii'}a_r\sigma(\inner{\mathbf{\hat{x}}_{i'}}{\boldsymbol{\theta}_r})
\end{align*}
As in previous convergence analysis for training neural networks, we assume that second-layer weights $\mathbf{a}$ are fixed and only the first layer weights $\boldsymbol{\Theta}$ are trainable.
Following the \texttt{GIST} feature partitioning strategy, we only partition the hidden layer.
Specifically, in global iteration $t$, sub-GCNs are constructed by sampling a set of masks $\mathcal{M}_t \in\R^{m\times d_1}$.
We denote the $j$th column of $\mathcal{M}_t$ as $\mathcal{M}^{(j)}_t\in\R^{m}$, the $r$th row of $\mathcal{M}_t$ as $\mathcal{M}_{t,r}\in\R^{d_1}$, and the entry in the $r$th row and $j$th column as $\mathcal{M}_{t,r}^{(j)}$.
Each $\mathcal{M}^{(j)}_{t,r}$ is a binary values: $\mathcal{M}^{(j)}_{t,r} = 1$ if neuron $r$ is active in sub-GCN $j$, and $\mathcal{M}^{(j)}_{t,r} = 0$ otherwise.
Using this mask notation, the output for node $i$ within sub-GCN $j$ can be written as
\begin{align*}
    \hat{y}_i^{(j)}(t,k) = f_{\mathcal{M}_{t}^{(j)}}(\boldsymbol{\Theta}_{t,k}^{(j)},\mathbf{X})_i=  \frac{1}{\sqrt{d_1}}\sum_{i'=1}^n\sum_{r=1}^{d_1}\bar{\mathbf{A}}_{ii'}\mathcal{M}_{t,r}^{(j)}a_r\sigma\left(\inner{\hat{\mathbf{x}_{i'}}}{\boldsymbol{\theta}_{t,k,r}^{(j)}}\right)
\end{align*}
$t$ and $k$ denote the current global and local iterations, respectively. %denotes the global iteration, while $k$
We assume that each $\mathcal{M}_{t,r}$ is sampled from a one-hot categorical distribution.
We formally define the random variables $\mathcal{M}_{t,r}^{(j)}$ as follows: Let each $\hat{m}_{t,r}$ be a uniform random variable on the index set $[m] = \{1,\dots,m\}$, i.e., $\mathcal{P}(\hat{m}_{t,r} = j) = \frac{1}{m}$ for $j\in [m]$.
Then, we define each mask entry as $\mathcal{M}_{t,r}^{(j)} = \I\{\hat{m}_{k,r} = j\}$.
Masks sampled in such a fashion have the following properties
\begin{itemize}
    \item $\mathcal{P}(\mathcal{M}_{t,r}^{(j)} = 1) = \frac{1}{m}$
    \item $\mathcal{P}(\mathcal{M}_{t,r}^{(j)} = 0) = 1 - \frac{1}{m}$
    \item $\sum_{j=1}^m\mathcal{M}_{t,r}^{(j)} = 1$
    \item $\mathcal{M}_{t,r}^{(j)}\mathcal{M}_{t,r}^{(j')} = 0$ if $j'\neq j$.
\end{itemize}
Here, the first and second properties guarantee that the expected number of neurons active in each sub-GCN is equal.
The third and fourth properties guarantee that each neuron is active in one and only one sub-GCN.
Within this setup, we consider the \texttt{GIST} training procedure, described as
\begin{equation}
    \label{GIST_procedure}
    \begin{aligned}
    \boldsymbol{\theta}_{t, 0, r}^{(j)} & = \boldsymbol{\theta}_{t, r}\\
    \boldsymbol{\theta}_{t, k+1, r}^{(j)} & = \boldsymbol{\theta}_{t, k, r}^{(j)} - \eta\frac{\partial L(\boldsymbol{\Theta}_{t,k}^{(j)})}{\partial\boldsymbol{\theta}_r}\\
    \boldsymbol{\theta}_{t+1, r} & = \boldsymbol{\theta}_{t,r} + \sum_{j=1}^m\left(\boldsymbol{\theta}_{t,\zeta,r}^{(j)} - \boldsymbol{\theta}_{t,0,r}^{(j)}\right)
\end{aligned}
\end{equation}
Within this formulation, $\zeta$ represents the total number of local iterations performed for each sub-GCN, while $L(\boldsymbol{\Theta}_{t,k}^{(j)})$ is the loss on the $j$th sub-GCN during the $t$th global and $k$th local iteration.
We can express $L(\boldsymbol{\Theta}_{t,k}^{(j)})$ as
\begin{align*}
    L\left(\boldsymbol{\Theta}_{t,k}^{(j)}\right) = \left\|\mathbf{y} - \hat{\mathbf{y}}^{(j)}(t,k)\right\|_2^2 = \left\|\mathbf{y} - f_{\mathcal{M}_{t}^{(j)}}(\boldsymbol{\Theta}_{t,k}^{(j)},\mathbf{X})\right\|_2^2
\end{align*}
and the gradient has the form
\begin{align*}
    \frac{\partial L(\boldsymbol{\Theta}_{t,k}^{(j)})}{\partial\boldsymbol{\theta}_r} = \frac{1}{\sqrt{d_1}}\sum_{i=1}^n\sum_{i'=1}^n\left(\hat{y}_i^{(j)}(t,k) - y_i\right)\bar{\mathbf{A}}_{ii'}\mathcal{M}_{t,r}^{(j)}a_r\mathbf{\hat{x}}_{i'}\I\left\{\inner{\boldsymbol{\theta}_{t,k,r}^{(j)}}{\mathbf{\hat{x}}_{i'}}\geq 0\right\}
\end{align*}

\subsection{Properties of the Transformed Input} \label{A:trans_input}
The GCN~\cite{gcn} uses a first-degree Chebyshev polynomial to approximate a spectral convolution on the graph, which results in an aggregation matrix of the form
\begin{equation}
    \label{Aggregation Matrix}
    \begin{aligned}
        \bar{\mathbf{A}} = \mathbf{I} + \mathbf{D}^{-\frac{1}{2}}\mathbf{A}\mathbf{D}^{-\frac{1}{2}}
    \end{aligned}
\end{equation}
where $\mathbf{A}$ is the adjacency matrix and $\mathbf{D}$ is the degree matrix with $\mathbf{D}_{ii} = \sum_{j=1}^n\mathbf{A}_{ij}$.
In practice, the re-normalization trick is applied to control the magnitude of the largest eigenvalue of $\bar{\mathbf{A}}$.
Here, however, we keep the original formulation of \eqref{Aggregation Matrix} to facilitate our analysis, and our assumption on the depth of the GCN does not lead to numerical instability even if $\lambda_{\max}(\bar{\mathbf{A}}) > 1$.
% However, in our work we stick with the original form in equation (\ref{Aggregation Matrix}), since such a form facilitate our analysis, and our assumption on the depth of the GCN does not lead to numerical instability even if $\lambda_{\max}(\bar{\mathbf{A}}) > 1$.
It is a well-known result that $2 = \lambda_{\max}(\bar{\mathbf{A})} \geq \lambda_{\min}(\bar{\mathbf{A}}) \geq 0$.
In particular, the lower bound on the minimum eigenvalue is obtained by considering
\begin{align*}
    \mathbf{v}^\top\bar{\mathbf{A}}\mathbf{v} = \sum_{i=1}^nv_i^2 + \sum_{(i,j)\in E}\frac{v_iv_j}{\sqrt{\mathbf{D}_{ii}\mathbf{D}_{jj}}} = \sum_{(i,j)\in E}\left(\frac{v_i}{\sqrt{\mathbf{D}_{ii}}} + \frac{v_j}{\sqrt{\mathbf{D}_{jj}}}\right)^2
\end{align*}
In our analysis, we require the aggregation matrix $\bar{\mathbf{A}}$ to be positive definite.
Thus, the following assumption can be made about $\lambda_{\min}(\mathbf{\bar{A}})$.
\begin{asump}
\label{agg_pd}
$\lambda_{\min}(\mathbf{\bar{A}}) \neq 0$.
\end{asump}
Going further, we must make a few more assumptions about the aggregation matrix and the graph itself to satisfy certain properties relevant to the analysis.
First, the following property must hold
% In this section we need to make more assumptions on the graph (in particular the aggregation matrix) such that the following properties hold.
\begin{property}
\label{bounded_norm}
For all $i\in[n]$, we have $\|\mathbf{\hat{x}}_i\|_2\leq 1$. 
\end{property}
% To make property (\ref{bounded_norm}) hold, we need the following condition
which can be guaranteed by the following assumption.
\begin{asump}
\label{balanced_degree}
There exists $\epsilon\in(0,1)$ and $p\in\mathbb{Z}_+$ such that
\begin{align*}
    (1-\epsilon)^2p\leq \mathbf{D}_{ii}\leq (1+\epsilon)^2p
\end{align*}for all $i\in[n]$.
\end{asump}
Additionally, we make the following assumption regarding the graph itself
\begin{asump}
\label{data_asump}
For all $i\in[n]$, we have $\|\mathbf{x}_i\|_2\leq \frac{1-\epsilon}{2}$, and $|y_i|\leq C$ for some constant $C$. Moreover, for all $j\in[n]$ and $j\neq i$, we have $\mathbf{x}_i\not\parallel \mathbf{x}_j$.
\end{asump}
which, in turn, yields the following property
\begin{property}
\label{none-parallel}
For all $i,j\in[n]$ such that $i\neq j$, we have $\hat{\mathbf{x}}_i\not\parallel\hat{\mathbf{x}}_j$.
\end{property}

\subsection{Full Statement Theorem \ref{main_gist_conv}} \label{A:gist_statement}
We now state the full version of theorem \ref{main_gist_conv} from Section \ref{S:theory}, which characterizes the convergence properties of one-hidden-layer GCN models trained with \texttt{GIST}.
The full proof of this Theorem is provided within Appendix \ref{A:gist_convergence}.
\begin{theorem}
Suppose assumptions \ref{agg_pd}-\ref{data_asump}, and property \ref{none-parallel} hold.
Moreover, suppose in each global iteration the masks are generated from a categorical distribution with uniform mean $\sfrac{1}{m}$.
Fix the number of global iterations to $T$ and local iterations to $\zeta$.
If the number of hidden neurons satisfies $d_1 = \Omega\left(\frac{n^3\zeta^2T^2}{\delta^2\gamma(1-\gamma)^2\lambda_0^4}\left(n + \frac{d}{m^2}\|\bar{\mathbf{A}}^2\|_{1,1}\right)\right)$, then procedure (\ref{GIST_procedure}) with constant step size $\eta = O\left(\frac{\lambda_0}{n^2\|\mathbf{A}^2\|_{1,1}}\right)$ converges according to
\begin{align*}
    \E_{[\mathcal{M}_{t-1}],\boldsymbol{\Theta}_0,\mathbf{a}}\left[\left\|\y - \hat{\y}(t)\right\|_2^2\right]& \leq\left(\gamma + (1-\gamma)\left(1 - \frac{\eta\lambda_0}{2}\right)^\zeta\right)^t\E_{\boldsymbol{\Theta}_0,\mathbf{a}}\left[\|\y - \hat{\y}(0)\|_2^2\right] + O\left(\frac{(m-1)^2\zeta\|\bar{\mathbf{A}}^2\|_{1,1}nd}{\gamma m^2d_1}\right)
\end{align*}
with probability at least $1 - \delta$.
\end{theorem}

\subsection{\texttt{GIST} and Local Training Progress}
For a one-hidden-layer MLP, the analysis often depends on the (scaled) Gram Matrix of the infinite-dimensional NTK
\begin{align*}
    \h^\infty_{ij} = \frac{1}{d_1m}\inner{\mathbf{\hat{x}}_i}{\mathbf{\hat{x}}_j}\E_{\boldsymbol{\theta}\sim\mathcal{N}(0,\mathbf{I})}\left[\I\{\inner{\mathbf{\hat{x}}_i}{\boldsymbol{\theta}}\geq 0, \inner{\mathbf{\hat{x}}_j}{\boldsymbol{\theta}}\geq 0\}\right]
\end{align*}
We can extend this definition of the Gram Matrix to an infinite-width, one-hidden-layer GCN as follows
\begin{align*}
    \mathbf{G}^{\infty} = \bar{\mathbf{A}}\h^\infty\bar{\mathbf{A}}
\end{align*}
With property \ref{none-parallel}, prior work \cite{du2019gradient} shows that $\lambda_{\min}(\h) > 0$. Denoting $\lambda_0 = \lambda_{\min}(\mathbf{G}^\infty)$, since $\bar{\mathbf{A}}$ is also positive definite, we have that $\lambda_0 \geq \lambda_{\min}(\h)\lambda_{\min}(\bar{\mathbf{A}}) > 0$. 
In our analysis, we define the Graph Independent Subnetwork Tangent Kernel (\texttt{GIST-K}) 
\begin{align*}
    \mathbf{G}^{(j)}(t, t', k) = \bar{\mathbf{A}}\h(t, t', k)\bar{\mathbf{A}}
\end{align*}
where $\h(t, t', k)$ is defined as 
\begin{align*}
    \h(t, t', k) = \frac{1}{d_1}\inner{\mathbf{\hat{x}}_i}{\mathbf{\hat{x}}_j}\sum_{r=1}^{d_1}\mathcal{M}_{t,r}^{(j)}\I\left\{\inner{\mathbf{\hat{x}}_i}{\boldsymbol{\theta}_{t',k,r}^{(j)}}\geq 0, \inner{\mathbf{\hat{x}}_j}{\boldsymbol{\theta}_{t',k,r}^{(j)}}\geq 0\right\}
\end{align*}
for masks $\mathcal{M}_t$ and weights $\boldsymbol{\Theta}_{t',k}^{(j)}$.
Following previous work \cite{liao2021convergence} on subnetwork theory, the following Lemma can be obtained.
\begin{lemma}
\label{PD_GISTK}
Suppose the number of hidden nodes satisfies $d_1 = \Omega\left(\lambda_0^{-1}n^2\log\sfrac{Tmn}{\delta}\right)$. If for all $t,k$ it holds that $\|\boldsymbol{\theta}_{t,k,r} - \boldsymbol{\theta}_{0,r}\|_2\leq R:= \frac{\lambda_0}{48n}$, then with probability at least $1-\delta$, for all $t,t'\in[T]$ we have: \vspace{-0.15cm}
\begin{align*}
    \lambda_{\min}(\mathbf{G}^{(j)}(t, t', k))\geq \tfrac{\lambda_0}{2}. \\[-15pt]
\end{align*}
\end{lemma}
After showing that every \texttt{GIST-K} is positive definite, we can then show that the local training of each sub-GCN enjoys a linear convergence rate.
\begin{lemma}
\label{local_converge}
Suppose the number of hidden nodes satisfies $d_1 = \Omega\left(\lambda_0^{-1}n^2\log\sfrac{Tmn}{\delta}\right)$. If for all $r\in[d_1]$ it holds that
\begin{equation}
\label{hypothesis1}
    \begin{aligned}
        \|\boldsymbol{\theta}_{t,r} - \boldsymbol{\theta}_{0,r}\|_2 + \frac{4T\eta\zeta}{\delta\alpha}\sqrt{\frac{n}{d_1}}\E_{[\mathcal{M}_{t-1}],\mathbf{W}_0,\mathbf{a}}\left[\|\y - \hat{\y}(t)\|_2^2\right]^\frac{1}{2} + (T - t)B\leq R
    \end{aligned}
\end{equation}with
\begin{align*}
    B = \sqrt{\frac{8(m-1)\eta\zeta n}{md_1}}\left(\sqrt{\frac{8(m-1)\|\bar{\mathbf{A}}^2\|_{1,1}d}{\gamma m}} + \sqrt{\frac{\eta\zeta nT}{\delta}}\right);\quad R\leq \frac{\lambda_0}{96n}
\end{align*}then we have
\begin{align*}
    \left\|\y - \hat{\mathbf{y}}^{(j)}(t,k+1)\right\|_2^2\leq \left(1 - \frac{\eta\lambda_0}{2}\right) \left\|\y - \hat{\mathbf{y}}^{(j)}(t,k)\right\|_2^2
\end{align*}
and for all $r\in [d_1], j\in[m]$ it holds that
\begin{align*}
    \left\|\boldsymbol{\theta}_{t,\zeta,r}^{(j)} - \boldsymbol{\theta}_{0,r}^{(j)}\right\|_2\leq \frac{2T\eta\zeta}{\delta}\sqrt{\frac{n}{d_1}}\E_{[\mathcal{M}_{t-1}],\mathbf{W}_0,\mathbf{a}}\left[\|\y - \hat{\y}(t)\|_2^2\right]^{\frac{1}{2}} + \eta\zeta n\sqrt{\frac{8(m-1)T}{md_1\delta}}
\end{align*}
with probability at least $1- \frac{\delta}{T}$
\end{lemma}

\subsection{Convergence of \texttt{GIST}} \label{A:gist_convergence}
We now prove the convergence result for \texttt{GIST} outlined in Appendix \ref{A:gist_statement}.
In showing the convergence of \texttt{GIST}, we care about the regression loss $\|\y - \hat{\y}(t)\|_2^2$ with
\begin{align*}
    \hat{\y}(t) = f(\boldsymbol{\Theta}_t,\mathbf{X}) = \frac{1}{m\sqrt{d_1}}\bar{\mathbf{A}}\sigma(\bar{\mathbf{A}}\mathbf{X}\boldsymbol{\Theta}_t)\mathbf{a}
\end{align*}
As in previous work \cite{liao2021convergence}, we add the scaling factor $\frac{1}{m}$ to make sure that $\E_{\mathcal{M}_t}[\hat{\y}^{(j)}(t,0)] = \hat{\y}(t)$. Moreover, by properties of the masks $\mathcal{M}_t^{(j)}$, we have
\begin{align*}
    f(\boldsymbol{\Theta},\mathbf{X}) = \sum_{j=1}^mf_{\mathcal{M}}^{(j)}(\boldsymbol{\Theta},\mathbf{X})
\end{align*}
Thus, we can invoke lemmas 13 and 14 from \cite{liao2021convergence}.
We state the two key lemmas here in accordance with our own notation.
\begin{lemma}
\label{global_error_decompose}
The $t$th global step produces squared error satisfying
\begin{align*}
    \|\y - \hat{\y}(t+1)\|_2^2 = \frac{1}{m}\sum_{j=1}^m\|\y - \hat{\y}^{(j)}(t,\zeta)\|_2^2 - \frac{1}{m^2}\sum_{j=1}^m\sum_{j'=1}^{j-1}\|\hat{\y}^{(j)}(t,\zeta) - \hat{\y}^{(j')}(t,\zeta)\|_2^2
\end{align*}
\end{lemma}
\begin{lemma}
\label{subnetwork_deviation}
In the $t$th global iteration, the sampled subnetwork's deviation from the whole network is given by
\begin{align*}
    \sum_{j=1}^m\|\hat{\y}(t) - \hat{\y}^{(j)}(t,0)\|_2^2 = \frac{1}{m}\sum_{j=1}^m\sum_{j'=1}^{j-1}\|\hat{\y}^{(j)}(t,0) - \hat{\y}^{(j')}(t,0)\|_2^2\\
\end{align*}
\end{lemma}
Moreover, lemmas 22 and 23 from \cite{liao2021convergence} show that with probability at least $1 - 2n\exp(-\frac{m}{32})$, for all $R\leq \frac{1}{2}$, it holds that
\begin{align*}
    \|\boldsymbol{\Theta}_0\|_F\leq \sqrt{2d_1d} - \sqrt{d_1}R\\
    \sum_{r=1}^m\inner{\boldsymbol{\theta}_{0,r}}{\hat{\mathbf{x}}_i}\leq d_1n(2 - R^2)
\end{align*}
For convenience, we assume that such an initialization property holds.
Then, we can use lemma 24 from \cite{liao2021convergence}: as long as $\|\theta_{t,r}- \theta_{0,r}\|_2\leq R$ for all $t,r$, then we have
\begin{align*}
    \E_{\mathcal{M}_t}\left[\|\hat{\y}(t) - \hat{\y}^{(j)}(t,0)\|_2^2\right]\leq \frac{4n(m-1)}{m^2}
\end{align*}
Then, applying Markov's inequality gives the following with probability at least $1-\frac{\delta}{2mT}$
% Therefore, apply Markov's inequality gives that with probability at least $1-\frac{\delta}{2mT}$
\begin{align*}
    \|\hat{\y}(t) - \hat{\y}^{(j)}(t,0)\|_2^2 \leq \frac{8n(m-1)T}{m\delta}
\end{align*}
We point out that, within the proof, we use $R = \frac{\lambda_0}{96n}$, which satisfies the condition above.
Using lemma \ref{global_error_decompose} to expand the loss at the $(t+1)$th iteration and invoking lemma \ref{local_converge} gives
\begin{align*}
    \left\|\y - \hat{\y}(t+1)\right\|_2^2 & = \frac{1}{m}\sum_{j=1}^m\left\|\y - \hat{\y}^{(j)}(t,\zeta)\right\|_2^2 - \frac{1}{m^2}\sum_{j=1}^m\sum_{j'=1}^{j-1}\left\|\hat{\y}^{(j)}(t,\zeta) - \hat{\y}^{(j')}(t,\zeta)\right\|_2^2\\
    & \leq \frac{1}{m}\sum_{j=1}^m\left(1 - \frac{\eta\lambda_0}{2}\right)^\zeta\left\|\y - \hat{\y}^{(j)}(t,0)\right\|_2^2- \frac{1}{m^2}\sum_{j=1}^m\sum_{j'=1}^{j-1}\left\|\hat{\y}^{(j)}(t,\zeta) - \hat{\y}^{(j')}(t,\zeta)\right\|_2^2\\
    & = \frac{1}{m}\sum_{j=1}^m\left\|\y - \hat{\y}^{(j)}(t,0)\right\|_2^2 - \frac{\eta\lambda_0}{2m}\sum_{k=0}^{\zeta-1}\sum_{j=1}^m\left(1 - \frac{\eta\lambda_0}{2}\right)^k\left\|\y - \hat{\y}^{(j)}(t,0)\right\|_2^2 - \\
    &\quad\quad\quad \frac{1}{m^2}\sum_{j=1}^m\sum_{j'=1}^{j-1}\left\|\hat{\y}^{(j)}(t,\zeta) - \hat{\y}^{(j')}(t,\zeta)\right\|_2^2\\
\end{align*}
Using the fact that $\E_{\mathcal{M}_t}[\hat{\y}^{(j)}(t,0)] = \hat{\y}(t)$ we have
\begin{align*}
    \E_{\mathcal{M}_t}\left[\left\|\y - \hat{\y}^{(j)}(t,0)\right\|_2^2\right] = \|\y - \hat{\y}(t)\|_2^2 + \E_{\mathcal{M}_t}\left[\left\|\hat{\y}(t)  - \hat{\y}^{(j)}(t,0)\right\|_2^2\right]
\end{align*}
Then, using lemma \ref{subnetwork_deviation} to rewrite the last term in the equation above and plugging in gives
\begin{align*}
    \E_{\mathcal{M}_t}\left[\left\|\y - \hat{\y}(t+1)\right\|_2^2\right]& \leq \left\|\y - \hat{\y}(t)\right\|_2^2 -  \frac{\eta\lambda_0}{2m}\sum_{k=0}^{\zeta-1}\sum_{j=1}^m\left(1 - \frac{\eta\lambda_0}{2}\right)^k\E_{\mathcal{M}_t}\left[\left\|\y - \hat{\y}^{(j)}(t,0)\right\|_2^2\right] + \\
    &\quad\quad\quad\frac{1}{m^2}\sum_{j=1}^m\sum_{j'=1}^{j-1}\E_{\mathcal{M}_t}\left[\|\hat{\y}^{(j)}(t,0) - \hat{\y}^{(j')}(t,0)\|_2^2-\left\|\hat{\y}^{(j)}(t,\zeta) - \hat{\y}^{(j')}(t,\zeta)\right\|_2^2\right]
\end{align*}
We denote the last term within the equation above as $\iota_t$
\begin{align*}
    \iota_t = \frac{1}{m^2}\sum_{j=1}^m\sum_{j'=1}^{j-1}\E_{\mathcal{M}_t}\left[\|\hat{\y}^{(j)}(t,0) - \hat{\y}^{(j')}(t,0)\|_2^2-\left\|\hat{\y}^{(j)}(t,\zeta) - \hat{\y}^{(j')}(t,\zeta)\right\|_2^2\right]
\end{align*}
The following lemma shows the bound on $\iota_t$
\begin{lemma}
\label{error_term_bound}
As long as $\left\|\boldsymbol{\theta}_{t,k,r}^{(j)} - \boldsymbol{\theta}_{0,r}\right\|_2\leq R$ for all $t,k,j$, and the initialization satisfies $\|\boldsymbol{\Theta}_0\|_F\leq \sqrt{2d_1d} - \sqrt{d_1}R$, then we have
\begin{align*}
    \iota_t &\leq \frac{\eta\gamma\lambda_0}{2m}\sum_{j=1}^n\sum_{k=0}^{\zeta-1}\E_{\mathcal{M}_t}\left[\|\y - \hat{\y}^{(j)}(t,k)\|_2^2\right] + \frac{64\eta(m-1)^2\zeta\|\bar{\mathbf{A}}^2\|_{1,1}nd}{\gamma m^2d_1} + \iota_t'
\end{align*}
with $\E_{\boldsymbol{\Theta}_0,\mathbf{a}_r}\left[\iota_t'\right]=0$, for all $\gamma\in(0,1)$.
\end{lemma}
Therefore, we can derive the following using lemma \ref{error_term_bound}
\begin{align*}
    \E_{\mathcal{M}_t}\left[\left\|\y - \hat{\y}(t+1)\right\|_2^2\right]& \leq \left\|\y - \hat{\y}(t)\right\|_2^2 -  \frac{\eta\lambda_0}{2m}\sum_{k=0}^{\zeta-1}\sum_{j=1}^m\left(1 - \frac{\eta\lambda_0}{2}\right)^k\E_{\mathcal{M}_t}\left[\left\|\y - \hat{\y}^{(j)}(t,0)\right\|_2^2\right] + \\
    &\quad\quad\quad\frac{\eta\gamma\lambda_0}{2m}\sum_{j=1}^n\sum_{k=0}^{\zeta-1}\E_{\mathcal{M}_t}\left[\|\y - \hat{\y}^{(j)}(t,k)\|_2^2\right] + \frac{64\eta(m-1)^2\zeta\|\bar{\mathbf{A}}^2\|_{1,1}nd}{\gamma m^2d_1} + \iota_t'\\
    & \leq \left\|\y - \hat{\y}(t)\right\|_2^2 -  \frac{(1-\gamma)\eta\lambda_0}{2m}\sum_{k=0}^{\zeta-1}\sum_{j=1}^m\left(1 - \frac{\eta\lambda_0}{2}\right)^k\E_{\mathcal{M}_t}\left[\left\|\y - \hat{\y}^{(j)}(t,0)\right\|_2^2\right] +\\
    &\quad\quad\quad\frac{64\eta(m-1)^2\zeta\|\bar{\mathbf{A}}^2\|_{1,1}nd}{\gamma m^2d_1}+\iota_t'\\
    & \leq \left\|\y - \hat{\y}(t)\right\|_2^2 -  \frac{(1-\gamma)\eta\lambda_0}{2}\sum_{k=0}^{\zeta-1}\left(1 - \frac{\eta\lambda_0}{2}\right)^k\left\|\y - \hat{\y}(t)\right\|_2^2 +\\
    &\quad\quad\quad\frac{64\eta(m-1)^2\zeta\|\bar{\mathbf{A}}^2\|_{1,1}nd}{\gamma m^2d_1}+\iota_t'\\
    & = \left(\gamma + (1-\gamma)\left(1 - \frac{\eta\lambda_0}{2}\right)^\zeta\right)\|\y - \hat{\y}(t)\|_2^2 +\frac{64\eta(m-1)^2\zeta\|\bar{\mathbf{A}}^2\|_{1,1}nd}{\gamma m^2d_1}+\iota_t'
\end{align*}

Starting from here, we use $\alpha$ to denote the global convergence rate
\begin{align*}
    \alpha = (1-\gamma)\left(1 - \left(1 - \frac{\eta\lambda_0}{2}\right)^\zeta\right)
\end{align*}Since $\zeta\geq 1$, we have that $\alpha\geq \frac{\eta\lambda_0}{2}(1-\gamma)$.
Then, the convergence rate above yields the following
\begin{align*}
    \E_{[\mathcal{M}_{t-1}],\boldsymbol{\Theta}_0,\mathbf{a}}\left[\left\|\y - \hat{\y}(t)\right\|_2^2\right]& \leq\E_{\boldsymbol{\Theta}_0,\mathbf{a}}\left[\|\y - \hat{\y}(0)\|_2^2\right] + O\left(\frac{(m-1)^2\zeta\|\bar{\mathbf{A}}^2\|_{1,1}nd}{\gamma m^2d_1}\right)
\end{align*}
Lastly, we provide a bound on weight perturbation using overparameterization. In particular, we can show that hypothesis \ref{hypothesis1} holds for iteration $t+1$
% Lastly, we bound the weight perturbation using overparameterization. In particular, we show that hypothesis (\ref{hypothesis1}) holds for iteration $t+1$:
\begin{align*}
    \|\boldsymbol{\theta}_{t+1,r} - \boldsymbol{\theta}_{0,r}\|_2 + \frac{4T\eta\zeta}{\delta\alpha}\sqrt{\frac{n}{d_1}}\E_{[\mathcal{M}_{t}],\boldsymbol{\Theta}_0,\mathbf{a}}\left[\|\y - \hat{\y}(t+1)\|_2^2\right]^\frac{1}{2} + (T-t - 1)B\leq R
\end{align*}
under the assumption that it holds in iteration $t$
\begin{align*}
    \|\boldsymbol{\theta}_{t,r} - \boldsymbol{\theta}_{0,r}\|_2 + \frac{4T\eta\zeta}{\delta\alpha}\sqrt{\frac{n}{d_1}}\E_{[\mathcal{M}_{t-1}],\boldsymbol{\Theta}_0,\mathbf{a}}\left[\|\y - \hat{\y}(t)\|_2^2\right]^\frac{1}{2} + (T - t)B\leq R
\end{align*}
and given the global convergence result
\begin{align*}
    \E_{\mathcal{M}_t}\left[\left\|\y - \hat{\y}(t+1)\right\|_2^2\right]& \leq \left(1 - \alpha\right)\|\y - \hat{\y}(t)\|_2^2 + \frac{64\eta(m-1)^2\zeta\|\bar{\mathbf{A}}^2\|_{1,1}nd}{\gamma m^2d_1} + \iota_t'
\end{align*}
Thus, it suffices to show that
\begin{align*}
    \|\boldsymbol{\theta}_{t+1,r} - \boldsymbol{\theta}_{0,r}\|_2 - \|\boldsymbol{\theta}_{t,r} - \boldsymbol{\theta}_{0,r}\|_2 & \leq \left(\E_{[\mathcal{M}_{t-1}],\boldsymbol{\Theta}_0,\mathbf{a}}\left[\|\y - \hat{\y}(t)\|_2^2\right]^\frac{1}{2} - \E_{[\mathcal{M}_{t}],\boldsymbol{\Theta}_0,\mathbf{a}}\left[\|\y - \hat{\y}(t+1)\|_2^2\right]^\frac{1}{2}\right) \cdot \\
    &\quad\quad\quad\frac{4T\eta\zeta}{\delta\alpha}\sqrt{\frac{n}{d_1}} + B
\end{align*}
Using Jensen's inequality, we derive the following
\begin{align*}
    \E_{[\mathcal{M}_{t}],\boldsymbol{\Theta}_0,\mathbf{a}}\left[\|\y - \hat{\y}(t+1)\|_2^2\right]^\frac{1}{2} & \leq\left(\left(1-\alpha\right)\E_{[\mathcal{M}_{t-1}],\boldsymbol{\Theta}_0,\mathbf{a}}\left[\|\y - \hat{\y}(t)\|_2^2\right] + \frac{64\eta(m-1)^2\zeta\|\bar{\mathbf{A}}^2\|_{1,1}nd}{\gamma m^2d_1}\right)^{\frac{1}{2}}\\
    & \leq \left(1 - \frac{\alpha}{2}\right)\E_{[\mathcal{M}_{t-1}],\boldsymbol{\Theta}_0,\mathbf{a}}\left[\|\y - \hat{\y}(t)\|_2^2\right]^{\frac{1}{2}} + \frac{8(m-1)}{m}\sqrt{\frac{\eta\zeta\|\bar{\mathbf{A}}^2\|_{1,1}nd}{\gamma d_1}}
\end{align*}
It then suffices to show that
\begin{align*}
    \|\boldsymbol{\theta}_{t+1,r} - \boldsymbol{\theta}_{0,r}\|_2 - \|\boldsymbol{\theta}_{t,r} - \boldsymbol{\theta}_{0,r}\|_2 & \leq \frac{2T\eta\zeta}{\delta}\sqrt{\frac{n}{d_1}} \E_{[\mathcal{M}_{t}],\boldsymbol{\Theta}_0,\mathbf{a}}\left[\|\y - \hat{\y}(t+1)\|_2^2\right]^\frac{1}{2} + \\
    &\quad\quad\quad B - \frac{8(m-1)}{m}\sqrt{\frac{\eta\zeta\|\bar{\mathbf{A}}^2\|_{1,1}nd}{\gamma d_1}}\\
    & = \frac{2T\eta\zeta}{\delta}\sqrt{\frac{n}{d_1}} \E_{[\mathcal{M}_{t}],\boldsymbol{\Theta}_0,\mathbf{a}}\left[\|\y - \hat{\y}(t+1)\|_2^2\right]^\frac{1}{2} + \\
    &\quad\quad\quad\eta\zeta n\sqrt{\frac{8(m-1)T}{md_1\delta}}
\end{align*}
Fix $r\in[d_1]$ and let $\hat{j}$ be the index of the sub-GCN in which $r$ is active. Indeed, we have
\begin{align*}
    \|\boldsymbol{\theta}_{t+1,r} - \boldsymbol{\theta}_{0,r}\|_2 & \leq \|\boldsymbol{\theta}_{t,r} - \boldsymbol{\theta}_{0,r}\|_2 + \|\boldsymbol{\theta}_{t+1,r}-\boldsymbol{\theta}_{0,r}\|_2\\
    & =\|\boldsymbol{\theta}_{t,r} - \boldsymbol{\theta}_{0,r}\|_2 + \|\boldsymbol{\theta}_{t,\zeta,r}^{(\hat{j})}-\boldsymbol{\theta}_{t,r}\|_2\\
    & \leq \|\boldsymbol{\theta}_{t,r} - \boldsymbol{\theta}_{0,r}\|_2 + \frac{2T\eta\zeta}{\delta}\sqrt{\frac{n}{d_1}}\E_{[\mathcal{M}_{t-1}],\boldsymbol{\Theta}_0,\mathbf{a}}\left[\|\y - \hat{\y}(t)\|_2^2\right]^{\frac{1}{2}} + \eta\zeta n\sqrt{\frac{8(m-1)T}{md_1\delta}}
\end{align*}
What remains is to prove hypothesis \ref{hypothesis1} for $t = 0$. In that case, we need
\begin{align*}
    \frac{4T\eta\zeta}{\delta\alpha}\sqrt{\frac{n}{d_1}}\E_{\boldsymbol{\Theta}_0,\mathbf{a}}\left[\|\y - \hat{\y}(0)\|_2^2\right]^\frac{1}{2} + B\leq R = O\left(\frac{\lambda_0}{n}\right)
\end{align*}
Finally, we have the following lemma bounding $\E_{\boldsymbol{\Theta}_0,\mathbf{a}}\left[\|\y - \hat{\y}(0)\|_2^2\right]$
\begin{lemma}
\label{init_error_bound}
It holds that
\begin{align*}
    \E\left[\|\y - \hat{\y}(0)\|_2^2\right]\leq C^2n + \frac{d}{m^2}\|\bar{\mathbf{A}}^2\|_{1,1}
\end{align*}
\end{lemma}
Thus, the bound above boils down to
\begin{align*}
    \frac{T\eta\zeta n}{\delta\alpha\sqrt{d_1}} = O\left(\frac{\lambda_0}{n}\right)\\
    \frac{T\eta\zeta}{\delta\alpha m}\sqrt{\frac{nd}{d_1}}\|\bar{\mathbf{A}}^2\|_{1,1}^{\frac{1}{2}}= O\left(\frac{\lambda_0}{n}\right)\\
    \frac{8(m-1)T}{m}\sqrt{\frac{\eta\zeta\|\bar{\mathbf{A}}^2\|_{1,1}nd}{\gamma d_1}} = O\left(\frac{\lambda_0}{n}\right)\\
    \eta\zeta n\sqrt{\frac{8(m-1)T^3}{md_1\delta}} = O\left(\frac{\lambda_0}{n}\right)
\end{align*}
Plugging in the value of $B$ and using $\alpha \geq \frac{\eta\lambda_0}{2}(1-\gamma)$ to solve for $d_1$ gives
\begin{align*}
    d_1 = \Omega\left(\frac{n^3\zeta^2T^2}{\delta^2\gamma(1-\gamma)^2\lambda_0^4}\left(n + \frac{d}{m^2}\|\bar{\mathbf{A}}^2\|_{1,1}\right)\right)
\end{align*}

\subsection{Proof of Lemmas}
We now provide all proofs for the major properties and lemmas utilized in deriving the convergence results for \texttt{GIST}.
\begin{proof}[Proof of Property \ref{bounded_norm}]
Under assumption \ref{balanced_degree}, we have that for all $i,i'\in[n]$
\begin{align*}
    \left(\frac{1-\epsilon}{1+\epsilon}\right)^2\leq\frac{\mathbf{D}_{ii}}{\mathbf{D}_{i'i'}}\leq \left(\frac{1 + \epsilon}{1 - \epsilon}\right)^2
\end{align*}
Therefore, we can write
\begin{align*}
    \|\hat{\mathbf{x}}_i\|_2 & = \left\|\sum_{i'=1}^n\bar{\mathbf{A}}_{ii'}\mathbf{x}_i\right\|_2\\
    & = \left\|\sum_{i'=1}^n\left(\mathbf{I} + \mathbf{D}^{-\frac{1}{2}}\mathbf{A}\mathbf{D}^{-\frac{1}{2}}\right)_{ii'}\mathbf{x}_i\right\|_2\\
    & = \left\|\mathbf{x}_i + \mathbf{D}_{ii}^{-\frac{1}{2}}\sum_{i'\neq i}\mathbf{A}_{ii'}\mathbf{D}^{-\frac{1}{2}}_{i'i'}\mathbf{x}_i\right\|_2\\
    & \leq \frac{1-\epsilon}{2} + \mathbf{D}_{ii}^{-\frac{1}{2}}\sum_{i'\neq i}\mathbf{A}_{ii'}\mathbf{D}^{-\frac{1}{2}}_{i'i'}\left(\frac{1-\epsilon}{2}\right)\\
    & \leq \frac{1-\epsilon}{2} + \mathbf{D}_{ii}^{-\frac{1}{2}}\sum_{i'\neq i}\mathbf{A}_{ii'}\mathbf{D}^{-\frac{1}{2}}_{ii}\left(\frac{1+\epsilon}{1-\epsilon}\right)\left(\frac{1-\epsilon}{2}\right)\\
    & = 1
\end{align*}
where the first inequality follows from assumption \ref{data_asump}.
\end{proof}

\begin{proof}[Proof of Lemma \ref{PD_GISTK}]
Fix some $R > 0$. Following Theorem 2 by \cite{liao2021convergence}, we have that with probability at least $1 - 2n^2e^{-2d_1t^2}$ it holds that
\begin{align*}
    \|\h^{(j)}(t,0,0) - \h^\infty\|_2 \leq nt
\end{align*}
and with probability at least $1 - n^2e^{-\frac{d_1R}{10m}}$ it holds that
\begin{align*}
    \|\h^{(j)}(t,t',k) - \h^{(j)}(k,0,0)\|_2 \leq \frac{3nR}{m}
\end{align*}
Choosing $t = \frac{\lambda_0}{16n}$ and $R = \frac{\lambda_0}{48n}$ gives
\begin{align*}
    \|\mathbf{G}^{(j)}(t,t',k) - \mathbf{G}^\infty\|_2 \leq \|\mathbf{\bar{A}}\|^2\|\h^{(j)}(t,t',k) - \h^\infty\|_2 = \|\mathbf{\bar{A}}\|^2\cdot\frac{\lambda_0}{8}\leq \frac{\lambda_0}{2}
\end{align*}
with probability at least $1 - n^2\left(2\exp\left(-\frac{d_1\lambda_0^2}{128n^2}\right) + \exp\left(-\frac{d_1\lambda_0}{480mn}\right)\right)$. Taking a union bound over all values of $t'$ and $j$, then plugging in the requirement $d_1 = \Omega\left(\lambda_0^{-1}n^2\log\sfrac{Tmn}{\delta}\right)$ gives the desired result.
\end{proof}

\begin{proof}[Proof of Lemma \ref{local_converge}]
We first bound the norm of the gradient as
\begin{align*}
    \left\|\frac{\partial L(\boldsymbol{\Theta}_{t,k}^{(j)})}{\partial\boldsymbol{\theta}_r}\right\|_2 \leq \frac{1}{\sqrt{d_1}}\sum_{i=1}^n\sum_{i'=1}^n\bar{\mathbf{A}}_{ii'}\left|\hat{y}_i^{(j)}(t,k) - y_i\right| = \frac{1}{\sqrt{d_1}}\|\bar{\mathbf{A}}\Delta\|_1\leq\sqrt{\frac{n}{d_1}}\|\bar{\mathbf{A}}\|\|\y - \hat{\y}^{(j)}(t,k)\|_2
\end{align*}
where here $\Delta = \left[\left|\hat{y}_1^{(j)}(t,k) - y_1\right|,\dots,\left|\hat{y}_n^{(j)}(t,k) - y_n\right|\right]$, and for the last inequality we use $\|\Delta\|_2 = \|\y - \hat{\y}^{(j)}(t,k)\|_2$.
Then, following \cite{song2020quadratic}, we first fix $R = \frac{\lambda_0}{\|\bar{\mathbf{A}}\|^2}$, and denote
\begin{align*}
    S_i & = \{r\in[m]: \neg A_{ir}\}\\
    A_{ir} & = \{\exists\boldsymbol{\theta}:\|\boldsymbol{\theta} - \boldsymbol{\theta}_{0,r}\|_2\leq R,\I\{\inner{\boldsymbol{\theta}}{\mathbf{\hat{x}}_i}\geq 0\}\neq \I\{\inner{\boldsymbol{\theta}_{0,r}}{\mathbf{\hat{x}}_i}\geq 0\}\}\\
    S_i^\perp & = [m]\setminus S_i\\
    \hat{s} & = \max_{i\in[n]}|S_i^\perp|
\end{align*}
Lemma 16 from \cite{liao2021convergence} shows that
\begin{align*}
    \mathcal{P}\left(|S_i^\perp|\leq4d_1R\right)\geq \exp(-d_1R)
\end{align*}
Throughout the proof, we let $\hat{s} = 4d_1R$. Moreover, we define 
\begin{align*}
    \h^\perp(t,t',k) = \frac{1}{d}\inner{\mathbf{\hat{x}}_i}{\mathbf{\hat{x}}_{i'}}\sum_{r\in S_i^\perp}\mathcal{M}_{t,r}^{(j)}\I\{\inner{\boldsymbol{\theta}_{t',k,r}^{(j)}}{\mathbf{\hat{x}}_i}\geq 0; \inner{\boldsymbol{\theta}_{t',k,r}^{(j)}}{\mathbf{\hat{x}}_{i'}}\geq 0\}
\end{align*}
and let $\mathbf{G}^\perp(t,t',k) = \bar{\mathbf{A}}\h^\perp(t,t',k)\bar{\mathbf{A}}$. Then, we have
\begin{align*}
    \left\|\h^\perp(t,t',k)\right\|_2^2 & \leq \left\|\h^\perp(t,t',k)\right\|_F^2\\
    & \leq \frac{1}{d_1^2}\sum_{i=1}^n\sum_{i'=1}^n\sum_{r\in S_i^\perp}\sum_{r'\in S_i^\perp}\I\left\{\inner{\boldsymbol{\theta}_{t',k,r}^{(j)}}{\mathbf{\hat{x}}_i}\geq 0; \inner{\boldsymbol{\theta}_{t',k,r}^{(j)}}{\mathbf{\hat{x}}_{i'}}\geq 0\right\}\cdot\\
    &\quad\quad\quad\I\left\{\inner{\boldsymbol{\theta}_{t',k,r'}^{(j)}}{\mathbf{\hat{x}}_i}\geq 0; \inner{\boldsymbol{\theta}_{t',k,r'}^{(j)}}{\mathbf{\hat{x}}_{i'}}\geq 0\right\}\\
    & \leq \frac{n^2\hat{s}^2}{d_1^2} = 16n^2R^2
\end{align*}
which yields the following
\begin{align*}
    \left\|\mathbf{G}^{(j)\perp}(t,t',k)\right\|\leq \|\bar{\mathbf{A}}\|^2\|\h^{(j)\perp}(t,t',k)\| = 16nR
\end{align*}
We then expand the loss at iteration $(t,k+1)$ as
\begin{align*}
    \left\|\y - \hat{\y}^{(j)}(t,k+1)\right\|_2^2 &= \left\|\y - \hat{\y}^{(j)}(t,k)\right\|_2^2 - 2\inner{\y - \hat{\y}^{(j)}(t,k)}{\hat{\y}^{(j)}(t,k+1) - \hat{\y}^{(j)}(t,k)} + \\
    &\quad\quad\quad\left\|\hat{\y}^{(j)}(t,k+1) - \hat{\y}^{(j)}(t,k)\right\|_2^2
\end{align*}
Starting to analyze the second term, we note that
\begin{align*}
    \hat{y}^{(j)}_i(t,k+1) - \hat{y}^{(j)}_i(t,k) = \frac{1}{\sqrt{d_1}}\sum_{i'=1}^n\sum_{r=1}^{d_1}\bar{\mathbf{A}}_{ii'}\mathcal{M}_{t,r}^{(j)}a_r\left(\sigma\left(\inner{\boldsymbol{\theta}_{t,k+1,r}^{(j)}}{\mathbf{\hat{x}}_{i'}}\right) - \sigma\left(\inner{\boldsymbol{\theta}_{t,k,r}^{(j)}}{\mathbf{\hat{x}}_{i'}}\right)\right)
\end{align*}
We decompose $\hat{y}^{(j)}_i(t,k+1) - \hat{y}^{(j)}_i(t,k) = I_{t,k,1}^{(j)} + I_{t,k,2}^{(j)}$ with
\begin{align*}
    I_{i,1}^{(j)}(t,k) = \frac{1}{\sqrt{d_1}}\sum_{i'=1}^n\sum_{r\in S_{i'}}\bar{\mathbf{A}}_{ii'}\mathcal{M}_{t,r}^{(j)}a_r\left(\sigma\left(\inner{\boldsymbol{\theta}_{t,k+1,r}^{(j)}}{\mathbf{\hat{x}}_{i'}}\right) - \sigma\left(\inner{\boldsymbol{\theta}_{t,k,r}^{(j)}}{\mathbf{\hat{x}}_{i'}}\right)\right)\\
    I_{i,2}^{(j)}(t,k) = \frac{1}{\sqrt{d_1}}\sum_{i'=1}^n\sum_{r\in S_{i'}^{\perp}}\bar{\mathbf{A}}_{ii'}\mathcal{M}_{t,r}^{(j)}a_r\left(\sigma\left(\inner{\boldsymbol{\theta}_{t,k+1,r}^{(j)}}{\mathbf{\hat{x}}_{i'}}\right) - \sigma\left(\inner{\boldsymbol{\theta}_{t,k,r}^{(j)}}{\mathbf{\hat{x}}_{i'}}\right)\right)
\end{align*}
where $I_{i,1}^{(j)}(t,k)$ can be further written as
\begin{align*}
    I_{i,1}^{(j)}(t,k) & = \frac{1}{\sqrt{d_1}}\sum_{i'=1}^n\sum_{r\in S_{i'}}\bar{\mathbf{A}}_{ii'}\mathcal{M}_{t,r}^{(j)}a_r\inner{\boldsymbol{\theta}_{t,k+1,r}^{(j)} - \boldsymbol{\theta}_{t,k,r}^{(j)}}{\mathbf{\hat{x}}_{i'}} \I\left\{\inner{\boldsymbol{\theta}_{t,k,r}^{(j)}}{\hat{\mathbf{x}}_{i'}}\geq 0\right\}\\
    & = -\frac{\eta}{\sqrt{d_1}}\sum_{i'=1}^n\sum_{r\in S_{i'}}\bar{\mathbf{A}}_{ii'}\mathcal{M}_{t,r}^{(j)}a_r\inner{\frac{\partial L\left(\boldsymbol{\Theta}^{(j)}_{t,k}\right)}{\partial\boldsymbol{\theta}_r}}{\mathbf{\hat{x}}_{i'}} \I\left\{\inner{\boldsymbol{\theta}_{t,k,r}^{(j)}}{\hat{\mathbf{x}}_{i'}}\geq 0\right\}\\
    & = \frac{\eta}{d_1}\sum_{i'=1}^n\sum_{i_1=1}^n\sum_{i_1'=1}^n\sum_{r\in S_i}\bar{\mathbf{A}}_{ii'}\bar{\mathbf{A}}_{i_1i_1'}\mathcal{M}^{(j)}_{t,r}\left(y_{i_1} - \hat{y}^{(j)}_{i_1}(t,k)\right)\inner{\hat{\mathbf{x}}_{i_1'}}{\hat{\mathbf{x}}_{i'}}\cdot\\
    &\quad\quad\quad\I\left\{\inner{\boldsymbol{\theta}_{t,k,r}^{(j)}}{\hat{\mathbf{x}}_{i'}}\geq 0;\inner{\boldsymbol{\theta}_{t,k,r}^{(j)}}{\hat{\mathbf{x}}_{i_1'}}\geq 0\right\}\\
    & = \eta\sum_{i'=1}^n\sum_{i_1=1}^n\sum_{i_1'=1}^n\bar{\mathbf{A}}_{ii'}\bar{\mathbf{A}}_{i_1i_1'}\left(y_{i_1} - \hat{y}^{(j)}_{i_1}(t,k)\right)\left(\h^{(j)}(t,t,k)_{i'i_1'} - \h^{(j)\perp}(t,t,k)_{i'i_1'}\right)
\end{align*}
Thus for $\mathbf{I}_{i,1}^{(j)}(t,k) = [I_{1,1}^{(j)}(t,k),\dots,I_{n,1}^{(j)}(t,k)]$ we have
\begin{align*}
    \mathbf{I}_{i,1}^{(j)}(t,k) & = \eta\bar{\mathbf{A}}\left(\h^{(j)}(t,t,k) - \h^{(j)\perp}(t,t,k)\right)\bar{\mathbf{A}}\left(\y - \hat{\y}^{(j)}(t,k)\right)\\
    & = \eta\left(\mathbf{G}^{(j)}(t,t,k) - \mathbf{G}^{(j)\perp}(t,t,k)\right)\left(\y - \hat{\y}^{(j)}(t,k)\right)\\
    & \geq \eta\left(\frac{\lambda_0}{2} - \|\mathbf{G}^{(j)\perp}(t,t,k)\|_2\right)\left\|\y - \hat{\y}^{(j)}(t,k)\right\|_2^2
\end{align*}
For $I_{i,2}^{(j)}(t,k)$ we have
\begin{align*}
    \left|I_{i,2}^{(j)}(t,k)\right| & \leq \frac{1}{\sqrt{d_1}}\sum_{i'=1}^n\sum_{r\in S_{i'}^\perp}\bar{\mathbf{A}}_{ii'}\left|\sigma\left(\inner{\boldsymbol{\theta}_{t,k+1,r}^{(j)}}{\mathbf{\hat{x}}_{i'}}\right) - \sigma\left(\inner{\boldsymbol{\theta}_{t,k,r}^{(j)}}{\mathbf{\hat{x}}_{i'}}\right)\right|\\
    & \leq \frac{\eta}{\sqrt{d_1}}\sum_{i'=1}^n\sum_{r\in S_{i'}^\perp}\bar{\mathbf{A}}_{ii'}\left\|\frac{\partial L(\boldsymbol{\Theta}_{t,k}^{(j)})}{\partial\boldsymbol{\theta}_r}\right\|_2\\
    & \leq \frac{\eta\hat{s}}{d_1}\|\bar{\mathbf{A}}\|\|\y - \hat{\y}^{(j)}(t,k)\|_2\sum_{i'=1}^n\bar{\mathbf{A}}_{ii'}\\
\end{align*}
which yields the following
\begin{align*}
    \left|\inner{\y - \hat{\y}^{(j)}(t,k)}{\mathbf{I}_{i,2}^{(j)}(t,k)}\right|& \leq \sum_{i=1}^n\left|y_i - y_i^{(j)}(t,k)\right|\cdot\left|I_{i,2}^{(j)}(t,k)\right|\\
    &\leq \frac{\eta\hat{s}}{d_1}\|\bar{\mathbf{A}}\|\|\y - \hat{\y}^{(j)}(t,k)\|_2\sum_{i,i'=1}^n\bar{\mathbf{A}}_{ii'}\left|y_i - y_i^{(j)}(t,k)\right|\\
    & \leq\frac{\eta\hat{s}}{d_1}\left\|\bar{\mathbf{A}}\right\|^2\left\|\y - \hat{\y}^{(j)}(t,k)\right\|_2^2\\
    & \leq \frac{4\eta\hat{s}}{d_1}\left\|\y - \hat{\y}^{(j)}(t,k)\right\|_2^2\\
\end{align*}
Lastly, we have
\begin{align*}
    \left(y_i - \hat{y}_i^{(j)}(t,k)\right)^2 & = \frac{1}{d_1}\sum_{i'=1}^n\sum_{i''=1}^n\sum_{r=1}^{d_1}\sum_{r'=1}^{d_1}\bar{\mathbf{A}}_{ii'}\bar{\mathbf{A}}_{ii''}\left\|\frac{\partial L(\boldsymbol{\Theta}_{t,k}^{(j)})}{\partial\boldsymbol{\theta}_r}\right\|_2\cdot\left\|\frac{\partial L(\boldsymbol{\Theta}_{t,k}^{(j)})}{\partial\boldsymbol{\theta}_{r'}}\right\|_2\\
    & \leq \eta^2\|\mathbf{\bar{A}}\|^2\|\y - \hat{\y}^{(j)}(t,k)\|_2^2\sum_{i'=1}^n\sum_{i''=1}^n\bar{\mathbf{A}}_{ii'}\bar{\mathbf{A}}_{ii''}
\end{align*}
Therefore
\begin{align*}
    \left\|\y - \hat{\y}^{(j)}(t,k)\right\|_2^2 &= \sum_{i=1}^n\left(y_i - \hat{y}_i^{(j)}(t,k)\right)^2\\
    & = \eta^2\|\mathbf{\bar{A}}\|^2\|\y - \hat{\y}^{(j)}(t,k)\|_2^2\sum_{i=1}^n\sum_{i'=1}^n\sum_{i''=1}^n\bar{\mathbf{A}}_{ii'}\bar{\mathbf{A}}_{ii''}\\
    & = 4\eta^2\|\mathbf{\bar{A}}^2\|_{1,1}\|\y - \hat{\y}^{(j)}(t,k)\|_2^2
\end{align*}
Putting things together gives
\begin{align*}
    \left\|\y - \hat{\y}^{(j)}(t,k+1)\right\|_2^2 & \leq \eta\left(2\|\mathbf{G}^{(j)\perp}(t,t,k)\|_2 + \frac{8\eta\hat{s}}{d_1} + 4\eta n^2\|\mathbf{\bar{A}}^2\|_{1,1} - \lambda_0\right)\|\y - \hat{\y}^{(j)}(t,k)\|_2^2 + \\
    &\quad\quad\quad\left\|\y - \hat{\y}^{(j)}(t,k)\right\|_2^2\\
    & \leq \left(1 - \frac{\eta\lambda_0}{2}\right)\left\|\y - \hat{\y}^{(j)}(t,k)\right\|_2^2
\end{align*}
where the last step follows by plugging in the values of $\|\mathbf{G}^{(j)\perp}(t,t,k)\|$ and $\hat{s}$, then setting $R = \frac{\lambda_0}{96n}$ and $\eta\leq \frac{\lambda_0}{n^2\|\mathbf{\bar{A}}^2\|_{1,1}}$. Next, we bound the weight perturbation. First, using Markov's inequality, we have that with probability at least $1 - \frac{\delta}{2T}$
\begin{align*}
    \|\y - \hat{\y}(t)\|_2^2 \leq \frac{2T}{\delta}\E_{[\mathcal{M}_{t-1}]}\left[\|\y - \hat{\y}(t)\|_2^2\right]
\end{align*}
Thus, we have
\begin{align*}
    \left\|\boldsymbol{\theta}_{t,k,r}^{(j)} - \boldsymbol{\theta}_{t,0,r}^{(j)}\right\|_2& \leq \sum_{k'=0}^{k-1}\left\|\boldsymbol{\theta}_{t,k+1,r} - \boldsymbol{\theta}_{t,k,r}\right\|_2\\
    & \leq \eta\sum_{k'=0}^{k-1}\left\|\frac{\partial L\left(\boldsymbol{\Theta}_{t,k}^{(j)}\right)}{\partial \boldsymbol{\theta}_r}\right\|_2\\
    & \leq \eta\sqrt{\frac{n}{d_1}}\sum_{k'=0}^{k-1}\|\y - \hat{\y}^{(j)}(t,k)\|_2\\
    & \leq \eta\zeta\sqrt{\frac{n}{d_1}}\left(\|\y - \hat{\y}(t)\|_2 +\|\hat{\y}(t) - \hat{\y}^{(j)}(t,0)\|_2\right)\\
    & \leq \frac{2T\eta\zeta}{\delta}\sqrt{\frac{n}{d_1}}\E_{[\mathcal{M}_{t-1}]}\left[\|\y - \hat{\y}(t)\|_2^2\right]^{\frac{1}{2}} + \eta\zeta n\sqrt{\frac{8(m-1)T}{md_1\delta}}
\end{align*}
Since $\alpha < 1 < 2$, we have that
\begin{align*}
    \frac{2T\eta\zeta}{\delta}\sqrt{\frac{n}{d_1}} \leq \frac{4T\eta\zeta}{\delta\alpha}\sqrt{\frac{n}{d_1}}
\end{align*}
Also we have
\begin{align*}
    \eta\zeta n\sqrt{\frac{8(m-1)T}{md_1\delta}} \leq B
\end{align*}
Therefore, by hypothesis \ref{hypothesis1}, we have that
\begin{align*}
    \left\|\boldsymbol{\theta}_{t,k,r}^{(j)} - \boldsymbol{\theta}_{0,r}^{(j)}\right\|_2\leq \left\|\boldsymbol{\theta}_{t,r}^{(j)} - \boldsymbol{\theta}_{0,r}^{(j)}\right\|_2 + \left\|\boldsymbol{\theta}_{t,k,r}^{(j)} - \boldsymbol{\theta}_{t,0,r}^{(j)}\right\|_2\leq R
\end{align*}
\end{proof}

\begin{proof}[Proof of Lemma (\ref{error_term_bound})]
For convenience, we denote
\begin{align*}
    \sigma_i^{(j)}(t,r) = \sigma\left(\inner{\boldsymbol{\theta}_{t,r}}{\hat{\mathbf{x}}_i}\right);\quad \sigma_i^{(j)}(t+1,r) = \sigma\left(\inner{\boldsymbol{\theta}_{t,\zeta,r}^{(j)}}{\hat{\mathbf{x}}_i}\right)
\end{align*}
Using $1$-Lipschitzness of ReLU, we have
\begin{align*}
    \left|\sigma_i^{(j)}(t,r) - \sigma_i^{(j)}(t+1,r)\right| & \leq \left\|\boldsymbol{\theta}_{t,r} - \boldsymbol{\theta}_{t,\zeta,r}^{(j)}\right\|_2\\
    & \leq \eta\sum_{k=0}^{\zeta - 1}\left\|\frac{\partial L(\boldsymbol{\Theta}_{t,k}^{(j)})}{\partial\boldsymbol{\theta}_{r}}\right\|_2\\
    & \leq 2\sqrt{\frac{n}{d_1}}\sum_{k=0}^{\zeta-1}\|\y - \hat{\y}^{(j)}(t,k)\|_2
\end{align*}
Also, we have
\begin{align*}
    \left|\sigma_i^{(j)}(t,r) + \sigma_i^{(j)}(t+1,r)\right| & \leq \left\|\boldsymbol{\theta}_{t,r} + \boldsymbol{\theta}_{t,\zeta,r}^{(j)}\right\|_2\\
    & \leq 2\|\boldsymbol{\theta}_{0,r}\|_2 + \left\|\boldsymbol{\theta}_{t,r} - \boldsymbol{\theta}_{0,r}\right\|_2 + \left\|\boldsymbol{\theta}_{t,\zeta,r}^{(j)} - \boldsymbol{\theta}_{0,r}\right\|_2\\
    & \leq 2\|\boldsymbol{\theta}_{0,r}\|_2 + 2R
\end{align*}
Expanding the difference of squares gives
\begin{align*}
    \left(\hat{y}_i^{(j)}(t,0) - \hat{y}_i^{(j')}(t,0)\right)^2 - \left(\hat{y}_i^{(j)}(t,\zeta) - \hat{y}_i^{(j')}(t,\zeta)\right)^2 = \beta_{i,1}\beta_{i,2}
\end{align*}
with
\begin{align*}
    \beta_{i,1}^{(j,j')} & = \hat{y}_i^{(j)}(t,0) - \hat{y}_i^{(j')}(t,0) + \hat{y}_i^{(j)}(t,\zeta) - \hat{y}_i^{(j')}(t,\zeta)\\
    \beta_{i,2}^{(j,j')} & = \hat{y}_i^{(j)}(t,0) - \hat{y}_i^{(j')}(t,0) - \hat{y}_i^{(j)}(t,\zeta) + \hat{y}_i^{(j')}(t,\zeta)
\end{align*}
Written in terms of the simplified notation, we have
\begin{align*}
    \beta_{i,1}^{(j,j')} & = \frac{1}{\sqrt{d_1}}\sum_{i'=1}^n\sum_{r=1}^{d_1}\bar{\mathbf{A}}_{ii'}a_r\left(\mathcal{M}_{t,r}^{(j)}\left(\sigma_{i'}^{(j)}(t,r) + \sigma_{i'}^{(j)}(t+1,r)\right) - \mathcal{M}_{t,r}^{(j')}\left(\sigma_{i'}^{(j')}(t,r) + \sigma_{i'}^{(j')}(t+1,r)\right)\right)\\
    \beta_{i,2}^{(j,j')} & = \frac{1}{\sqrt{d_1}}\sum_{i'=1}^n\sum_{r=1}^{d_1}\bar{\mathbf{A}}_{ii'}a_r\left(\mathcal{M}_{t,r}^{(j)}\left(\sigma_{i'}^{(j)}(t,r) - \sigma_{i'}^{(j)}(t+1,r)\right) - \mathcal{M}_{t,r}^{(j')}\left(\sigma_{i'}^{(j')}(t,r) - \sigma_{i'}^{(j')}(t+1,r)\right)\right)
\end{align*}
Letting 
\begin{align*}
    \tau_1(i',r) & = \mathcal{M}_{t,r}^{(j)}\left(\sigma_{i'}^{(j)}(t,r) + \sigma_{i'}^{(j)}(t+1,r)\right) - \mathcal{M}_{t,r}^{(j')}\left(\sigma_{i'}^{(j')}(t,r) + \sigma_{i'}^{(j')}(t+1,r)\right)\\
    \tau_2(i',r) & = \mathcal{M}_{t,r}^{(j)}\left(\sigma_{i'}^{(j)}(t,r) - \sigma_{i'}^{(j)}(t+1,r)\right) - \mathcal{M}_{t,r}^{(j')}\left(\sigma_{i'}^{(j')}(t,r) - \sigma_{i'}^{(j')}(t+1,r)\right)
\end{align*}
Then we have
\begin{align*}
    \beta_{i,1}^{(j,j')}\beta_{i,2}^{(j,j')} & = \frac{1}{d_1}\sum_{i_1=1}^n\sum_{i_2=1}^n\sum_{r=1}^{d_1}\sum_{r'=1}^{d_1}\bar{\mathbf{A}}_{ii_1}\bar{\mathbf{A}}_{ii_2}a_ra_{r'}\tau_1(i_1,r)\tau_2(i_2,r')\\
    & = \frac{1}{d_1}\sum_{i_1=1}^n\sum_{i_2=1}^n\sum_{r=1}^{d_1}\bar{\mathbf{A}}_{ii_1}\bar{\mathbf{A}}_{ii_2}\tau_1(i_1,r)\tau_2(i_2,r) + \Delta_{i,t}^{(j,j')}
\end{align*}
with 
\begin{align*}
    \Delta_{i,t}^{(j,j')} = \frac{1}{d_1}\sum_{i_1=1}^n\sum_{i_2=1}^n\sum_{r=1}^{d_1}\sum_{r'\neq r}\bar{\mathbf{A}}_{ii_1}\bar{\mathbf{A}}_{ii_2}a_ra_{r'}\tau_1(i_1,r)\tau_2(i_2,r')
\end{align*}
Note that due to the independence of $a_r$ and $a_{r'}$, we have that $ \E_{\mathbf{a}}\left[\Delta_{i,t}^{(j,j')}\right] = 0$. Moreover, for $j\neq j'$, we have that $\mathcal{M}_{t,r}^{(j)}\mathcal{M})_{t,r}^{(j')} = 0$. Thus, we have
\begin{align*}
    \left|\tau(i_1,r)\tau_2(i_2,r')\right| &= \mathcal{M}_{t,r}^{(j)}\left|\sigma_{i_1}^{(j)}(t,r) + \sigma_{i_1}^{(j)}(t+1,r)\right|\cdot\left|\sigma_{i_2}^{(j)}(t,r) - \sigma_{i_2}^{(j)}(t+1,r)\right| + \\
    &\quad\quad\quad\mathcal{M}_{t,r}^{(j')}\left|\sigma_{i_1}^{(j')}(t,r) + \sigma_{i_1}^{(j')}(t+1,r)\right|\cdot\left|\sigma_{i_2}^{(j')}(t,r) - \sigma_{i_2}^{(j')}(t+1,r)\right|\\
    & \leq 4\eta\sqrt{\frac{n}{d_1}}\left(\|\boldsymbol{\theta}_{0,r}\|_2 + 2R\right)\left(\sum_{k=0}^{\zeta-1}\|\y - \hat{\y}^{(j)}(t,k)\|_2 + \sum_{k=0}^{\zeta-1}\|\y - \hat{\y}^{(j')}(t,k)\|_2\right)
\end{align*}
Thus
\begin{align*}
    \beta_{i,1}^{(j,j')}\beta_{i,2}^{(j,j')} & = \frac{4\eta\sqrt{n}}{d_1^{\frac{3}{2}}}\left(\sum_{i_1=1}^n\sum_{i_2=1}^n\bar{\mathbf{A}}_{ii_1}\bar{\mathbf{A}}_{ii_2}\right)\left(\sum_{r=1}^{d_1}\left(\|\boldsymbol{\theta}_{0,r}\|_2 + 2R\right)\right)\cdot\\
    &\quad\quad\quad\left(\sum_{k=0}^{\zeta-1}\|\y - \hat{\y}^{(j)}(t,k)\|_2 + \sum_{k=0}^{\zeta-1}\|\y - \hat{\y}^{(j')}(t,k)\|_2\right) + \Delta_{i,t}^{(j,j')}\\
    & \leq 4\eta\sqrt{\frac{2nd}{d_1}}\left(\sum_{i_1=1}^n\sum_{i_2=1}^n\bar{\mathbf{A}}_{ii_1}\bar{\mathbf{A}}_{ii_2}\right)\left(\sum_{k=0}^{\zeta-1}\|\y - \hat{\y}^{(j)}(t,k)\|_2 + \sum_{k=0}^{\zeta-1}\|\y - \hat{\y}^{(j')}(t,k)\|_2\right)+ \Delta_{i,t}^{(j,j')}
\end{align*}
Thus
\begin{align*}
    \iota_t & =\frac{1}{m^2}\sum_{j=1}^m\sum_{j'=1}^{j-1}\E_{\mathcal{M}_t}\left[\sum_{i=1}^n\beta_{i,1}^{(j,j')}\beta_{i,2}^{(j,j')}\right]\\
    &= \frac{4\eta(m-1)}{m^2}\sqrt{\frac{2nd}{d_1}}\|\bar{\mathbf{A}}^2\|_{1,1}\sum_{j=1}^m\sum_{k=0}^{\zeta-1}\E_{\mathcal{M}_t}\left[\|\y - \hat{\y}^{(j)}(t,k)\|_2\right] + \frac{1}{m^2}\sum_{j=1}^m\sum_{j'=1}^{j-1}\sum_{i=1}^n\E_{\mathcal{M}_t^{(j)}}\left[\Delta_{i,t}^{(j,j')}\right]\\
    & = \sum_{j=1}^n\sum_{k=0}^{\zeta-1}\E_{\mathcal{M}_t}\left[\alpha_{t,k}^{(j)}\right] + \frac{1}{m^2}\sum_{j=1}^m\sum_{j'=1}^{j-1}\sum_{i=1}^n\E_{\mathcal{M}_t^{(j)}}\left[\Delta_{i,t}^{(j,j')}\right]
\end{align*}
with
\begin{align*}
    \alpha_{t,k}^{(j)} = \frac{4\eta(m-1)}{m^2}\sqrt{\frac{2nd}{d_1}}\|\bar{\mathbf{A}}^2\|_{1,1}\|\y - \hat{\y}^{(j)}(t,k)\|_2
\end{align*}
By Cauchy-Schwartz inequality,
\begin{align*}
    \alpha_{t,k}^{(j)} & = \frac{\eta}{m}\cdot\left(\sqrt{\frac{\gamma\lambda_0}{2}}\|\y - \hat{\y}^{(j)}(t,k)\|_2\right)\cdot\left(\frac{8(m-1)}{m\sqrt{\gamma}}\|\bar{\mathbf{A}}\|_{1,1}\sqrt{\frac{nd}{d_1}}\right)\\
    & \leq \frac{\eta}{m}\left(\frac{\gamma\lambda_0}{2}\|\y - \hat{\y}^{(j)}(t,k)\|_2^2 +\frac{64(m-1)^2\|\bar{\mathbf{A}}\|_{1,1}^2nd}{\gamma m^2d_1}\right)\\
    & = \frac{\eta\gamma\lambda_0}{2m}\|\y - \hat{\y}^{(j)}(t,k)\|_2^2 + \frac{64\eta(m-1)^2\|\bar{\mathbf{A}}\|_{1,1}^2nd}{\gamma m^3d_1}
\end{align*}
Thus
\begin{align*}
    \iota_t &\leq \frac{\eta\gamma\lambda_0}{2m}\sum_{j=1}^m\sum_{k=0}^{\zeta-1}\E_{\mathcal{M}_t}\left[\|\y - \hat{\y}^{(j)}(t,k)\|_2^2\right] + \frac{64\eta(m-1)^2\zeta\|\bar{\mathbf{A}}\|_{1,1}^2nd}{\gamma m^2d_1} + \iota_t'
\end{align*}
with
\begin{align*}
    \E_{\boldsymbol{\Theta}_0,\mathbf{a}}\left[\iota_t'\right] = 0
\end{align*}
\end{proof}

\begin{proof}[Proof of Lemma \ref{init_error_bound}]
Note that
\begin{align*}
    \E_{\boldsymbol{\Theta}_0,\mathbf{a}}\left[\|\y - \hat{\y}(0)\|_2^2\right] & = \sum_{i=1}^n\E_{\boldsymbol{\Theta}_0,\mathbf{a}}\left[(y_i - \hat{y}_i(0))^2\right]\\
    & =y_i^2 -2y_i\E_{\boldsymbol{\Theta}_0,\mathbf{a}}\left[\hat{y}_i(0)\right] + \E_{\boldsymbol{\Theta}_0,\mathbf{a}}\left[\hat{y}_i(0)^2\right]\\
    & \leq C^2 + \E_{\boldsymbol{\Theta}_0,\mathbf{a}}\left[\hat{y}_i(0)^2\right]
\end{align*}
where the last inequality follows from the bound on $|y_i|$ and the fact that $\E_{\boldsymbol{\Theta}_0,\mathbf{a}}\left[\hat{y}_i(0)^2\right] = 0$.
Moreover, we have
\begin{align*}
    \E_{\boldsymbol{\Theta}_0,\mathbf{a}}\left[\hat{y}_i(0)^2\right] & = \frac{1}{m^2 d_1}\sum_{i_1=1}^n\sum_{i_2=1}^n\sum_{r_1=1}^{d_1}\sum_{r_2=1}^{d_1}\bar{\mathbf{A}}_{ii_1}\bar{\mathbf{A}}_{ii_2}\E_{\mathbf{a}}[a_{r_1}a_{r_2}]\E_{\boldsymbol{\Theta}_0}\left[\sigma\left(\inner{\boldsymbol{\theta}_{0,r_1}}{\hat{\mathbf{x}}_{i_1}}\right)\sigma\left(\inner{\boldsymbol{\theta}_{0,r_2}}{\hat{\mathbf{x}}_{i_2}}\right)\right]\\
    & = \frac{1}{m^2d_1}\sum_{i_1=1}^n\sum_{i_2=1}^n\sum_{r=1}^{d_1}\bar{\mathbf{A}}_{ii_1}\bar{\mathbf{A}}_{ii_2}\E_{\boldsymbol{\Theta}_0}\left[\sigma\left(\inner{\boldsymbol{\theta}_{0,r}}{\hat{\mathbf{x}}_{i_1}}\right)\sigma\left(\inner{\boldsymbol{\theta}_{0,r}}{\hat{\mathbf{x}}_{i_2}}\right)\right]\\
    & \leq \frac{1}{m^2d_1}\sum_{i_1=1}^n\sum_{i_2=1}^n\sum_{r=1}^{d_1}\bar{\mathbf{A}}_{ii_1}\bar{\mathbf{A}}_{ii_2}\E_{\boldsymbol{\Theta}_0}\left[\|\boldsymbol{\theta}_{0,r}\|_2^2\right]\\
    & = \frac{d}{m^2}\sum_{i_1=1}^n\sum_{i_2=1}^n\bar{\mathbf{A}}_{ii_1}\bar{\mathbf{A}}_{ii_2}
\end{align*}
Thus
\begin{align*}
     \E_{\boldsymbol{\Theta}_0,\mathbf{a}}\left[\|\y - \hat{\y}(0)\|_2^2\right] \leq \sum_{i=1}^n\E_{\boldsymbol{\Theta}_0,\mathbf{a}}\left[\hat{y}_i(0)^2\right]\leq  C^2n + \frac{d}{m^2}\|\bar{\mathbf{A}}^2\|_{1,1}
\end{align*}
\end{proof}
\end{document}